\documentclass[11pt,reqno,a4paper]{amsart}
\usepackage[toc,page]{appendix}
\pdfoutput=1
\usepackage{geometry}
\newcommand{\e}{\mathrm{e}}
\usepackage[T1]{fontenc}
\usepackage{comment}
\usepackage[utf8]{inputenc}
\usepackage{MnSymbol}
\usepackage{esint}
\usepackage{enumitem}
\usepackage[english]{babel}
\usepackage{amsmath}
\usepackage{systeme}
\usepackage{thm-restate}
\usepackage{float}
\usepackage{bbm}
\usepackage[normalem]{ulem}
\usepackage{graphicx}
\usepackage{mathtools}
\usepackage{enumitem}
\usepackage{tikz}
\usepackage{pgfplots}
\usepackage{lineno}
\usepackage{cancel}
\usepackage{subfigure}
\pgfplotsset{compat = newest}

\usepackage[
colorlinks,
pdfpagelabels,
pdfstartview = FitH,
bookmarksopen = true,
bookmarksnumbered = true,
linkcolor = blue,
plainpages = false,
hypertexnames = false,
citecolor = red] {hyperref}

\hypersetup{
linktoc=page
}

\usepackage[capitalize]{cleveref}
\crefname{enumi}{}{}
\crefname{equation}{}{}

\makeatletter
\def\@tocline#1#2#3#4#5#6#7{\relax
\ifnum #1>\c@tocdepth 
\else
\par \addpenalty\@secpenalty\addvspace{#2}%
\begingroup \hyphenpenalty\@M
\@ifempty{#4}{%
\@tempdima\csname r@tocindent\number#1\endcsname\relax
}{%
\@tempdima#4\relax
}%
\parindent\z@ \leftskip#3\relax \advance\leftskip\@tempdima\relax
\rightskip\@pnumwidth plus4em \parfillskip-\@pnumwidth
#5\leavevmode\hskip-\@tempdima
\ifcase #1
\or\or \hskip 1em \or \hskip 2em \else \hskip 3em \fi%
#6\nobreak\relax
\dotfill\hbox to\@pnumwidth{\@tocpagenum{#7}}\par
\nobreak
\endgroup
\fi}
\makeatother

\newtheorem{theorem}{Theorem}
\newtheorem*{theorem*}{Theorem}
\newtheorem{proposition}{Proposition}[section]
\newtheorem{lemma}[proposition]{Lemma}
\newtheorem{corollary}[proposition]{Corollary}

\theoremstyle{definition}
\newtheorem{definition}[proposition]{Definition}
\newtheorem{remark}[proposition]{Remark}
\newtheorem{example}[proposition]{Example}
\numberwithin{equation}{section}

\def \R {\mathbb {R}}

\def \E {\mathbb{E}}
\def \P {\mathbb{P}}

\def \W {\mathrm{W}}

\def\diam{\operatorname{diam}}
\def\dist{\operatorname{dist}}

\renewcommand{\epsilon}{\varepsilon}

\newcommand{\dz}{\, {\rm d} z}
\newcommand{\dt}{\, {\rm d} t}
\newcommand{\dy}{\, {\rm d} y}

\newcommand{\dV}{\, {\rm d} V}

\newcommand{\dvar}[1]{\, {\rm d} #1}
\newcommand{\logamma}[2]{\gamma \left(#1,#2 \right)}

\newcommand{\interior}{\operatorname{Int}}

\newcommand\Set[2]{\{\,#1\mid#2\,\}}

\DeclareFontFamily{U}{mathx}{\hyphenchar\font45}
\DeclareFontShape{U}{mathx}{m}{n}{
<5> <6> <7> <8> <9> <10>
<10.95> <12> <14.4> <17.28> <20.74> <24.88>
mathx10
}{}
\DeclareSymbolFont{mathx}{U}{mathx}{m}{n}
\DeclareFontSubstitution{U}{mathx}{m}{n}
\DeclareMathAccent{\widecheck}{0}{mathx}{"71}

\newcommand{\norm}[1]{\left\lVert#1\right\rVert}

\allowdisplaybreaks

\begin{document}
\title[]{Exploring singularities in data with the graph Laplacian: An explicit approach}
\email{\color{blue} benny.avelin@math.uu.se}
\author[Andersson]{Martin Andersson}
\address{Martin Andersson,
  Department of Mathematics,
  Uppsala University,
  S-751 06 Uppsala,
  Sweden}
\email{\color{blue} martin.andersson@math.uu.se}
\author[Avelin]{Benny Avelin}
\address{Benny Avelin,
  Department of Mathematics,
  Uppsala University,
  S-751 06 Uppsala,
  Sweden}

\begin{abstract}
  We develop theory and methods that use the graph Laplacian to analyze the geometry of the underlying manifolds of datasets. Our theory provides theoretical guarantees and explicit bounds on the functional forms of the graph Laplacian when it acts on functions defined close to singularities of the underlying manifold.
  We use these explicit bounds to develop tests for singularities and propose methods that can be used to estimate geometric properties of singularities in the datasets.
\end{abstract}
\subjclass{Primary 58K99; Secondary 68R99, 60B99.}
\keywords{Graph Laplacian, geometry, singularities}

\maketitle

\section{Introduction}
High dimensional data is common in many research problems across academic fields, and such data can often be represented as points collected or sampled from $\R^N$. A common assumption is that the dataset $X = \{X_i\}_i^n$ $ \subset \mathbb{R}^N$ lies on a lower-dimensional set $\Omega$ and is in fact a sample from some probability distribution over $\Omega$. A further assumption, that makes our model of data more tractable, is that $\Omega$ can be represented as the union of several well-behaved manifolds, i.e. $\Omega = \cup_i \Omega_i$. Here, each $\Omega_i$ could represent a different class in a classification problem: For instance, if a dataset contains two classes, $i$ and $j$, class $i$ might be contained in $\Omega_i$ and class $j$ in $\Omega_j$, with the two classes potentially being disjoint. However, classification is not always so clear-cut: For instance, in the MNIST dataset, handwritten digits of $"1" \in \Omega_1$ and $"7"\in \Omega_7$ can appear very similar, suggesting that $\Omega_1 \cap \Omega_7 \neq \emptyset$. Therefore, understanding geometric situations such as intersections is of interest in classification problems.

In our manifold model of data, an intersection between two different manifolds $\Omega_i, \Omega_j$ is either represented just as such, or it can be viewed as a singularity if we consider $\Omega = \Omega_i \cup \Omega_j$ as a single manifold. Other  regions in $\Omega$ that can be viewed as singular, such as boundaries and edges, may also be of interest as they can signify important features in the data. Regions that are part of just one manifold $\Omega_i$, and in its interior, we consider as non-singular.

To study such singularities, we use the graph Laplacian $L_{n,t}$. This operator, which depends on the number of data points $n$ and a parameter $t$, can act on functions defined on the dataset $X$. In non-singular regions, as $n$ tends to infinity and $t$ tends to 0, $L_{n,t}$ converges to the Laplace-Beltrami operator~\cite{belkin2008towards}. In this work, we primarily study the behavior of $x \to L_{n,t}f(x)$ for functions $f$, when $x$ is close to singular points.

\subsection{Motivation}
Consider data points situated in the ambient space, where the underlying manifold structure is unknown. We examine the graph Laplacian operator applied to the linear function $f(x) = v \cdot x$, where $v$ is a unit vector. On our discrete dataset, this function reduces to the vector $\mathbf{f} = (v \cdot X_1, \ldots, v \cdot X_n) \in \R^n$. The contribution of this paper is that we calculate explicitly how the graph Laplacian acts on such linear functions $f$, which proves to be valuable for estimation purposes.

The choice of $f$ is motivated by the convergence properties of $L_{n,t}$ to the Laplace-Beltrami operator. In the interior of $\Omega$, we expect that $L_{n,t}f(x) \approx 0$. However, for singular points like intersections, the limit operator is of first order~\cite{BQWZ}, and $L_{n,t}f(x) \neq 0$, which can be seen in \cref{fig:lap}.

Our results show how $x \to L_{t}f(x)$ and, through a finite-sample bound, how $x \to L_{n,t}f(x)$ behaves. More specifically, given $x_0 \in \Omega_i$ near some singularity, and $x$ in the ball $B_R(x_0)$, including the case when $x \not \in \Omega_i$, we show how the function $x \to L_{n,t}f(x)$ deviates from being constantly 0 and has specific  functional forms. These forms depend on the type of singularity.

\subsection{Overview of results}
In \cref{sec:flat}, we consider the scenario where $\Omega = \cup_{i=1}^m \Omega_i$ is a union of compact $d$-dimensional flat (linear) submanifolds of $\R^N$. This geometric configuration, illustrated in \cref{fig:thm2}, is particularly relevant to the neural network architecture discussed in \cref{sec:neural}.

To set up the results, we start with an $x_0 \in \Omega$ which is assumed to be a flat manifold, and let $x \in B_{R}(x_0)$, where $R=\sqrt{t}r_0>0$, and use $\hat x$ to denote the projection of $x$ to the tangent space of $\Omega$. We also define $v_{n,\Omega}$ as the projection of $v$ onto $x-\hat x$, and $v_{n,\partial \Omega}$ is the projection of $v$ onto the outwards normal of $\partial \Omega$, which is assumed to be flat (locally linear) close to $x$. Then we show the following:
\begin{itemize}
  \item In \cref{thm:explicit:intersection}, we let $\| x- x_0 \| = r\sqrt{t}$ and $\theta$ is the angle between vectors $x-x_0$ and $\hat x - x_0$. If $x$ is not close to $\partial \Omega$, then
        \begin{equation*}
          L_t f(x) = A(x)t^{\frac{d+1}{2}} v_{n,\Omega} \sin\theta\, r {\e}^{-\sin^2\theta\, r^2} + t^{\frac{d+1}{2}}B(x) {\e}^{-r_0^2}.
        \end{equation*}
        The function $A$ is bounded between $\frac{p}{2}\pi^{d/2}$ and $p\pi^{d/2}$, where $p$ is the uniform density on $\Omega$, and $B$ is uniformly bounded. Note that for small $t$ we have $r_0$ can be taken large and thus the last term is very small.
  \item \cref{thm:explicit:boundary} shows what happens when $x$ is close to $\partial \Omega$:
        \begin{align*}
          L_t f(x) & =\widehat{A}_1(x) t^{\frac{d+1}{2}} v_{n,\Omega} \sin\theta\, r {\e}^{-\sin^2\theta\, r^2}                                            \\
                   & \qquad + \widehat{A}_2(x) t^{\frac{d+1}{2}} v_{n,\partial \Omega} {\e}^{-\sin^2\theta\, r^2} \quad + t^{\frac{d+1}{2}} B(x) {\e}^{-r_0^2},
        \end{align*}
        where functions $\widehat A_1, \widehat A_2$ and $B$ have explicitly computable bounds.
\end{itemize}
In \cref{sec:general} and \cref{sec:noise} we prove more general results:
\begin{itemize}
  \item In \cref{thm:general} we relax the conditions on $\Omega$, considering non-flat manifolds, and prove a weaker version of \cref{thm:explicit:intersection}.
  \item In \cref{thm:noisy} we relax the conditions further, and allow for noise when sampling from $\Omega$.
\end{itemize}
In \cref{sec:hypothesis} we show how the results can be used to construct hypothesis tests for intersections in data. We test this method in \cref{sec:num:hypo}. We also show how such hypothesis tests can be used to detect singularities in zero sets of neural networks in \cref{sec:num:neural}.

Finally, in \cref{sec:num:sing} we propose methods to find intersections in data and estimate the angle of such intersections, which are motivated by the aforementioned theorems and \cref{cor:intersection}. We also provide numerical experiments, in \cref{sec:numerical}, to test these methods.

\section{Earlier work}

\begin{figure}[ht]
  \includegraphics[width=12cm]{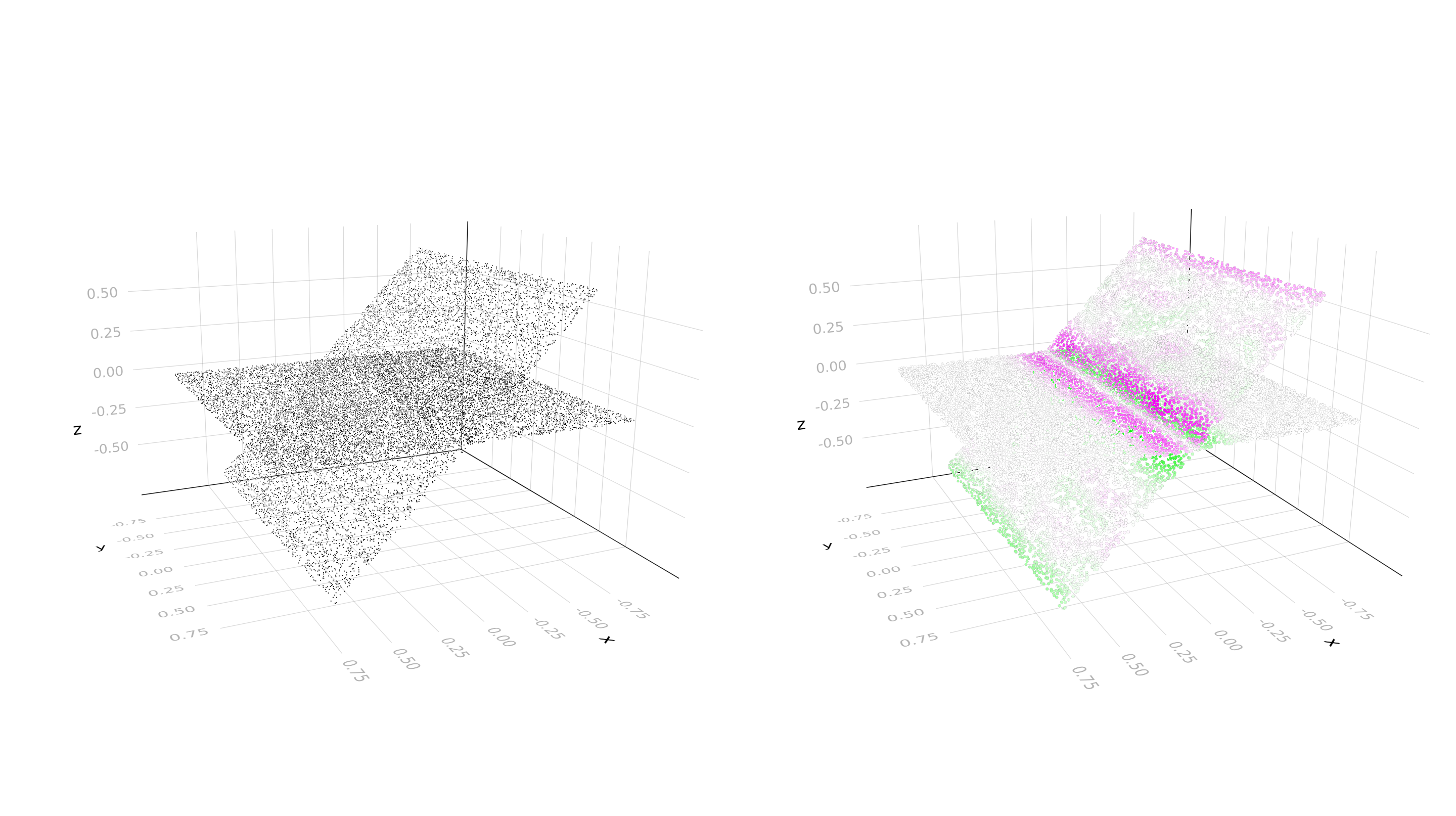}
  \centering
  \caption{Graph Laplacian $L_{n,t}$ acting on a linear function $f$. Purple color showing positive, and green color negative values of $L_{n,t}f$, where lack of color indicates values near 0}
  \label{fig:lap}
\end{figure}
The framework of assuming an underlying low-dimensional manifold of data, in conjunction with graph-related tools and, in particular, the graph Laplacian, has been used extensively. Some examples include work in clustering~\cite{belkin2001laplacian,kannan2004clusterings,luxburg2004convergence,shi2000normalized,ng2001spectral}, dimensionality reduction~\cite{belkin2003laplacian,nadler2006diffusion}, and semi-supervised learning~\cite{belkin2004semi}.

Several of the approaches to study datasets that use the graph Laplacian leverage that if the manifold is smooth enough and well-behaved, then the graph Laplacian approximates some well-understood operator (for instance, the Laplace-Beltrami operator~\cite{belkin2006convergence}), which has useful mathematical properties.

Therefore, convergence properties of the graph Laplacian becomes important, and it has been studied in~\cite{belkin2008towards,bousquet2003measure,lafon2004diffusion,nadler2006diffusion}. In particular, and highly influential of this paper, is what the asymptotic convergence looks like near singularities of the manifold, which was shown in~\cite{BQWZ}.

The special case that $\Omega$ is composed of linear manifolds has received earlier attention. For instance, trying to recover facts about the underlying geometry has been explored in~\cite{vidal2005generalized,chen2009foundations}. However, these methods have not utilized the graph Laplacian, but other means.

\subsection{Main motivating application} \label{sec:neural}
There is a lot of interest in understanding the geometry of the loss landscape of neural networks. In particular, zero sets of the loss function.

As a motivating example for the methodology we develop in this paper, consider the following problem.
Consider the following class of neural networks:
\begin{equation}\label{eq:net_sets}
  f_W(x) = \sum_{i=1}^k a_i (w_i \cdot x)_+, \quad x \in \mathbb{S}^1,\, a_i = \pm 1,\, w_i \in \R^2
\end{equation}
where $\mathbb{S}^1$ is the unit circle. For general spheres this is what is called a spherical neural network and has in essence all the same properties as a standard single hidden layer network. This class of networks has been studied in for instance~\cite{arora2019fine,du2018gradient,avelin2020approximation}. Let a target function $g(x) = f_{W^\ast}(x)$ be given, but with unknown parameter $W^\ast$. We are interested in the set of parameters $\Omega_\delta$ such that
\begin{equation*}
  \Omega_\delta := \{W \in \R^{2k}: |f_W(x) - g(x)| < \delta, \quad x \in \mathbb{S}^1 \}.
\end{equation*}
We know by the definition of such networks that $\Omega_\delta$ is the sublevelset of a piecewise linear function. It is thus clear that the sublevelset is itself a union of linear manifolds.

\section*{Declarations of interest}
None.

\section{Basic mathematical objects and theory}
In this section, we provide more precise definitions and introduce the basic mathematical theory we will be using to present and prove our results.

\subsection{Conditions on manifolds}\label{sec:manifolds}
We will consider sets of the form $\Omega = \cup_i^m \Omega_i$, where each $\Omega_i$ is a smooth and compact $d$-dimensional Riemannian submanifold of $\mathbb{R}^N$. We will assume that if any pair $\Omega_i,\Omega_j$, and $i\neq j$, have a non-empty intersection, then this intersection will have dimension lower than $d$.

Our analysis will assume that we only have access to samples from $\Omega$, and so associated to $\Omega$ there is a probability measure with uniform density $p: \Omega \to \R$. Relaxing this would not significantly alter the main arguments, but the analysis would become more cumbersome.

For any point $x \in \Omega_i$ in the interior, we can consider the tangent space $T_{\Omega_i,x} \simeq \R^d$, which we will identify as a subspace of the ambient space $\R^N$. More precisely, given open subsets $U\subset \R^d$ and $W \subset \Omega_i$ ($W$ is open in the subspace topology of $\Omega_i$), and a coordinate chart $\alpha: U \to W$ such that $\alpha(0) = x$, we define $T_{\Omega_i,x}$ as the image of $\R^d$ under the action of the Jacobian. We denote the Jacobian $D\alpha:U \to \R^{N \times d}$, evaluated at 0, by $D\alpha(0)$. The best linear approximation to $u \mapsto \alpha(u)$ is, of course, given by $u \mapsto x + D\alpha(0)u$, and $x + T_{\Omega_i,x}$ is a first order approximation to $\Omega_i$ at $x$.

The definition of $\Omega$ implies that an interior point $x\in \Omega$ can have more than one associated tangent space. For example, if $x\in \Omega_i \cap \Omega_j$ and $i \neq j$, then both $T_{\Omega_i,x}$ and $T_{\Omega_j,x}$ exist, and can be different.

A note on notation is that we denote the interior of a manifold $\Omega_i$ by $\interior \Omega_i$, and the boundary by $\partial \Omega_i$.

\subsection{Types of singularities}\label{sec:singularities}

The following are what we will refer to as singular points, which will be of four different kinds. Given $x \in \Omega = \cup \Omega_i$, we have the following types:

\newlist{Types}{enumerate}{10}
\setlist[Types]{label=\textbf{(Type \arabic*)},ref=Type \arabic*,leftmargin=*}
\begin{Types}
  \item \label{it:type1}  There is a submanifold $\Omega_i$ such that $x \in \partial \Omega_i$.
  \item \label{it:type2} There are submanifolds $\Omega_i \neq \Omega_j$ such that $x \in \interior \Omega_i \cap \interior \Omega_j$.
  \item \label{it:type3} There are submanifolds $\Omega_i \neq \Omega_j$ such that $x\in \partial \Omega_i \cap \interior \Omega_j$.
  \item \label{it:type4} There are submanifolds $\Omega_i \neq \Omega_j$ such that $x \in \partial \Omega_i \cap \partial \Omega_j$.
\end{Types}

The different types above can of course have non-empty intersection with each other, and a \emph{non-singular} point is simply a point $x\in \interior \Omega_i$ such that if $j \neq i$, then $x \not \in \Omega_j$. See \cref{fig:XV-config} for two examples of singularities.

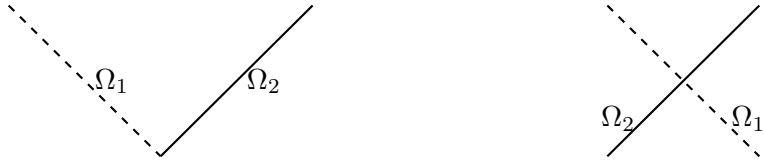
\begin{figure}[!ht]
  \begin{minipage}{0.45\textwidth}
    \centering
    \begin{tikzpicture}
      \draw[dashed,thick] (-2,2) -- (0,0);
      \draw[thick] (0,0) -- (2,2);
      \node[anchor=west] at (-1,1) {$\Omega_1$};
      \node[anchor=west] at (1,1) {$\Omega_2$};
    \end{tikzpicture}
  \end{minipage}
  \begin{minipage}{0.45\textwidth}
    \centering
    \begin{tikzpicture}[scale=0.5]
      \draw[dashed,thick] (-2,2) -- (2,-2);
      \draw[thick] (-2,-2) -- (2,2);
      \node[anchor=east] at (-1,-1) {$\Omega_2$};
      \node[anchor=west] at (1,-1) {$\Omega_1$};
    \end{tikzpicture}
  \end{minipage}
  \caption{There is a singularity in the intersection of the lines above. The left figure shows a point of~\ref{it:type4}, and the right figure shows a point of~\ref{it:type2}.}
  \label{fig:XV-config}
\end{figure}

\subsubsection{$(L,r)$-regular manifolds}
To formulate our results for general non-flat manifolds, we need a quantitative measure of how regular, with regard to curvature, our set $\Omega$ is.

Recall (see \cref{sec:integration}) that for any $x_0 \in \interior \Omega_i$, there is a neighborhood $W \subset \Omega_i$ of $x_0$ on which the orthogonal projection $\pi$ to the affine tangent plane $x_0 + T_{\Omega_i, x_0}$ is a diffeomorphism onto its image.

\begin{definition}\label{def:Lr-regular}
  Let $\Omega = \bigcup_i \Omega_i$ be a finite union of compact $d$-dimensional submanifolds in $\R^N$. We say that $\Omega$ is \emph{$(L,r)$-regular} if for each $\Omega_i$, and every $x_0 \in \interior \Omega_i$ with $\dist(x_0, \partial \Omega_i) \geq r$, the orthogonal projection $\pi: W \to x_0 + T_{\Omega_i, x_0}$ is smooth in a neighborhood $W$ of $x_0$ in $\Omega_i$ with $B_r(x_0) \cap \Omega_i \subset W$, and for all $y \in W$:
  \begin{equation}\label{eq:manifold:bound1}
    \norm{y - \pi(y)} \leq L \norm{x_0 - \pi(y)}^2,
  \end{equation}
  and
  \begin{equation}\label{eq:manifold:bound2}
    |V(D(\pi^{-1})(\pi(y))) - 1| \leq L \norm{x_0 - \pi(y)}^2,
  \end{equation}
  where $V$ is the volume function from \cref{sec:integration}.
\end{definition}

\begin{remark}
  Condition~\eqref{eq:manifold:bound1} bounds the deviation of $\Omega_i$ from its tangent plane at $x_0$, while condition~\eqref{eq:manifold:bound2} bounds the distortion of the volume element under the chart. Both are second-order conditions that hold automatically for smooth manifolds with bounded curvature.
\end{remark}

\begin{example}
  Any smooth compact submanifold $\Omega_i \subset \R^N$ is $(L,r)$-regular for some $L, r > 0$. This follows from compactness: the bounds~\eqref{eq:manifold:bound1} and~\eqref{eq:manifold:bound2} hold locally at each point in $\interior \Omega_i$, and compactness allows us to extract uniform constants.
\end{example}

\subsection{Graph Laplacian}
In this section we introduce the graph Laplacian and how it acts on real-valued functions defined on $\R^N$.

Given  $n$ i.i.d.~random samples $X = \{X_1,\ldots,X_n\}$ from the distribution with density $p$ on $\Omega$, we build a weighted fully connected graph $G=(V,E)$ as follows: We let each sample $X_i$ represent a vertex $i$, and for vertices $i,j \in V $ the weight on $(i,j) \in E$ is given by
\begin{equation*}
  W_{n,t}(i,j) := W_{n,t}(X_i,X_j) = \frac{1}{n} K_t(X_i,X_j) = \frac{1}{n} {\e}^{-\frac{\|X_i - X_j\|^2}{t}}.
\end{equation*}
The function $W_{n,t}$ is naturally viewed as an $n\times n$ matrix, and the variable $t$ is in the literature often referred to as the \emph{bandwidth} of the \emph{kernel} $K_t$.
\begin{remark}
  In the limit analysis as $n\to \infty$, it is useful also normalize by $\frac{1}{t^{\frac{d+1}{2}}}$. But, since a priori we do not know the dimension $d$, we will work without this normalization.
\end{remark}
We define the diagonal weighted degree matrix as
\begin{equation*}
  D_{n,t}(i,i) = \sum_j W_{n,t}(i,j),
\end{equation*}
and the \emph{graph Laplacian} $L_{n,t}$ as
\begin{equation*}
  L_{n,t} = D_{n,t} - W_{n,t}.
\end{equation*}
\begin{remark}
  This is often referred to as the \emph{unnormalized graph Laplacian}. There are other normalizations of this matrix which are used, for example, in~\cite{shi2000normalized,ng2001spectral,lafon2004diffusion}. One difference between these normalizations are their limit properties.
\end{remark}
Given the fully connected graph $G = (V,E)$, the graph Laplacian above can be seen as an operator acting on arbitrary functions $f:V \to \R$ in the following way:
\begin{equation*}
  L_{n,t} f(X_i) = \frac{1}{n} \sum_{j} K_t(X_i,X_j) (f(X_i) - f(X_j)), \quad (X_i,X_j) \in E.
\end{equation*}
We extend this operator to acting on functions that live in the ambient space $f:\R^N \to \R$ by the canonical choice
\begin{equation} \label{eq:Lnt}
  L_{n,t} f(x) = \frac{1}{n} \sum_{j} K_t(x,X_j) (f(x) - f(X_j)), \quad x \in \R^N.
\end{equation}

Several of our results will be stated in terms of the expected operator:
\begin{equation}
  \label{eq:Lt}
  L_t f(x) = \E_p[L_{n,t}f(x)] = \int_\Omega K_t(x,y) (f(x) - f(y)) p(y)\dy.
\end{equation}
Further, an immediate consequence of the linearity of the integral is that
\begin{equation}\label{eq:Lt_sum}
  \begin{split}
    L_{t}f(x) & = \int_\Omega K_t(x,y)(f(x)-f(y))p(y)\dy                  \\
              & = \sum_i \int_{\Omega_i} K_t(x,y)(f(x)-f(y))p(y) \dy.
  \end{split}
\end{equation}
This is useful to us, and in our approach we work with the \emph{restricted Laplacian} $L_t^i$, which we define as
\begin{equation} \label{eq:restricted_laplacian}
  L_{t}^if(x) = \int_{\Omega_i} K_t(x,y)(f(x)-f(y))p(y)\dy.
\end{equation}

\subsection{The class of functions}
As mentioned in the introduction, we apply the graph Laplacian to linear functions of the form $f(x) = v \cdot x$, where $v$ is a unit vector, and all our results are stated with this assumption.
\subsection{A concentration inequality}
As stated earlier, many of our results pertain to the expected operator $L_{t}$. To connect these results to the operator $L_{n,t}$, we prove the following concentration inequality:
\begin{theorem}\label{thm:finite:sample}
  Let $X_1,\ldots,X_n$ be i.i.d.~samples from a density $p$ on the union of manifolds $\Omega = \cup \Omega_i$. Then
  \begin{equation*}
    \P \left ( \max_i \left |L_{n,t}f(X_i) - \frac{n-1}{n} L_t f(X_i) \right | > \epsilon \right ) \leq 2 n \exp\left ( -\frac{4\e (n-1) \epsilon^2}{t} \right ).
  \end{equation*}
\end{theorem}

This concentration can be improved if we have some further information about the manifold. For instance, if we know that the manifold is flat, and we know the dimension. This allows us to get tight control over the variance of the graph Laplacian, which is what improves the concentration, but we do not explore this further.

\section{Main results}

Now that we have the necessary definitions and mathematical background, we are ready to present our main results. Before stating the theorems, we will provide a brief section that explains the geometry of some terms that will be used.

\subsection*{General structure of results}
By \cref{eq:Lt_sum} it is enough to understand the restricted Laplacian, $L^i_t$ defined in \cref{eq:restricted_laplacian}. Because of this, our results are formulated to show the behavior of $L^i_{t}$. Depending on what type of singularity being examined, it is straight forward to extend the results to the full Laplacian. In \cref{cor:intersection} we give one example  of how to extend the results to the sum $\sum_{i=1}^2 L^i_t$ when one is close to an intersection of two manifolds.

\subsubsection*{Geometry and notation for \cref{sec:flat}}\label{sec:geom_notation}
We will in several theorems also formulate the function $x \to L_t^i f(x)$ partly in terms of new coordinates $(r,\theta)$. Here $r$ is defined by the relation $\norm{x-x_0} = \sqrt{t}r$, and given the projection $\hat x$ of $x$ to the tangent plane of the flat manifold $\Omega_i$, we define $\theta \in [0,\pi/2]$ to be the angle between vectors $x-x_0$ and $\hat x - x_0$, as in the schematic in \cref{fig:thm2}. By simple geometry, it follows that $\norm{\hat x - x} = r\sqrt{t}\sin \theta$.

Given a vector $v \in \R^N$, we will have reason to write the expression $ v \cdot (x - \hat x)$ as
\begin{equation*}
  v \cdot (x - \hat x) = r\sqrt{t}\sin\theta \, v \cdot \frac{x - \hat x}{\norm{x - \hat x}} = r\sqrt{t}\sin\theta \, v_{n,\Omega_i}(x),
\end{equation*}
where we have defined
\begin{equation*}
v_{n,\Omega_i}(x) \coloneqq
\begin{cases}
v \cdot \frac{x - \hat{x}}{|x - \hat{x}|} & \text{if } x \neq \hat{x} \\
0 & \text{if } x = \hat{x}.
\end{cases}
\end{equation*}
Here $v_{n,\Omega_i}$ is the projection of $v$ onto a unit normal vector of $\Omega_i$, and which can vary with $x$. We will later use a natural coordinate system parallel to $\Omega_i$, see \cref{sec:integration}, such that when $x\neq \hat x$, this function is constant up to its sign. This implies that evaluating $ r\sqrt{t}\sin\theta \, v_{n,\Omega_i}(x)$ is the same as letting $v_{n,\Omega_i}$ be fixed, but allowing $\theta$ to change sign depending on which side of $\Omega_i$ $x$ is, i.e.~as if we have fixed the coordinate system in which we measure the angle $\theta$. In the following, to increase readability, we suppress the $x$-dependency of $v_{n,\Omega_i}$.

Additionally, in \cref{thm:explicit:boundary} we will have a term $v_{n,\partial \Omega_i}$ that is specific to that theorem. This will be defined in the case where there is a boundary close to $x$. In \cref{fig:thm2}, this would imply there is a boundary of $\Omega_1$ nearby. To give the definition of this term first assume that $\Omega_i$ is a flat manifold and assume that close to the point $x$, $\partial \Omega_i$ is flat. Then, we first let $\hat x_{\partial \Omega_i}$ be the projection of $\hat x$ to $\partial \Omega_i$. We can now define a unit normal at $\hat x_{\partial \Omega_i}$, denoted by  $n_{\partial \Omega_i}$. Two choices are natural, a normal pointing either towards, or away from $\Omega_i$. We define $n_{\partial \Omega_i}$ as the latter. Given a vector $v \in \R^N$, we define
\begin{equation*}
  v_{n,\partial \Omega} \coloneqq  v \cdot n_{\partial \Omega}.
\end{equation*}
Since we assume that close to $x$, $\partial \Omega$ is flat, $n_{\partial \Omega}$ can be considered constant.

\begin{figure}[ht]
  \begin{tikzpicture}[scale=2]
    \draw[gray, thick] (-1,0) -- (4,0);
    \draw[gray, thick] (0.5,-1) -- (2.5,1);
    \filldraw (1.5,0) circle (0.05) node[below] {$x_0$};
    \filldraw (2.5,1) circle (0.05) node[right] {$x$};
    \draw[dashed] (2.5,0) -- (2.5,1) node[midway,right]{$\sin \theta r \sqrt{t}$};
    \filldraw (2.5,0) circle (0.05) node[below] {$\hat x$};
    \draw[dashed] (1.5+0.4,0) arc (0:45:0.4);
    \draw (1.5+0.5,0+0.2) node{$\theta$};

    \draw[dashed] (1.5,0) -- (1.5,1.41) node[midway]{$r \sqrt{t}$};
    \draw[dashed] (1.5,0) circle (1.41);

    \draw (0.5,0) node[above]{$\Omega_1$};
    \draw (0.5,-1+0.1) node[above]{$\Omega_2$};

    \draw[dashed] (3,-1.41) arc (-45:45:2);
    \draw[dashed] (0,-1.41) arc (45+180:-45+180:2);
    \draw[dashed] (1.5,0) -- (1.5-1.41,1.41) node[midway]{$r_0 \sqrt{t}$};
  \end{tikzpicture}
  \caption{Schematic picture of the geometry of \cref{thm:explicit:intersection}, where $\Omega_1$ is the object of interest and $x \in \Omega_2$ for visualization purposes.}
  \label{fig:thm2}
\end{figure}
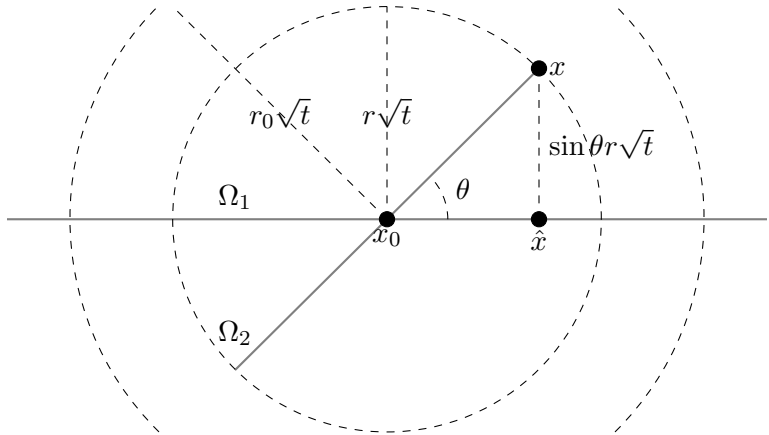

\subsubsection*{Geometry and notation for \cref{sec:general}}\label{sec:general_geom_notation}
To help with the geometric picture for general manifolds, the situation is as explained in \cref{sec:geom_notation} and \cref{fig:thm2}: the terms $x,x_0, \hat x, \theta$ and $v_{n,\Omega_i}$ are in the same relation to each other as in \cref{sec:geom_notation}. The main difference is that the tangent space now changes with $x$ as the manifold is no longer flat. In that sense the geometry for more general manifolds is not more difficult, but handling error terms is more involved. When considering multiple manifolds, we write $\theta_i$ (or $\theta_i(x)$ when emphasizing the dependence on $x$) for the angle $\theta$ computed with respect to the tangent plane of $\Omega_i$.

\subsection{Flat manifolds}\label{sec:flat}

In this section, in addition to the assumptions presented in \cref{sec:manifolds}, we assume that each $\Omega_i$ is a flat manifold, i.e. a $d$ dimensional compact linear submanifold of $\R^N$.
This implies that each coordinate chart around a point $x \in \interior \Omega_i$ is an isometry with an open neighborhood $U$ of $x$, where $U$ is a ball in $\R^d$.

In \cref{thm:explicit:intersection} we give a result concerning the behavior of $x\to L_t^if(x)$ when we are \emph{not} close to the boundary $\partial \Omega$. This case is easier to prove, and we give explicit bounds of all terms involved, and express them with elementary functions.

In \cref{thm:explicit:boundary} we show what happens when we are close to $\partial \Omega$, and here our results contain more complex expressions for some terms.

In the following theorems, we think of $x_0$ as the potentially singular point, see \cref{fig:thm2}. The appearance of terms different from zero either comes from the fact that we are evaluating the restricted operators outside of their corresponding manifolds, or because we are close to the boundary. For example, if $\Omega = \Omega_1 \cup \Omega_2$ and $x_0 \in \Omega_1 \cap \Omega_2$, the extra contribution thus comes from each $L_{t}^i$, $i \in \{1,2\}$, when evaluated on the other manifolds, i.e. $L_t^1 f(x)$ when $x \in \Omega_2$ and vice versa. By combining \cref{thm:explicit:intersection} and \cref{thm:explicit:boundary}, it is possible to consider several types of singularities defined in \cref{sec:singularities}.

\begin{theorem}
  \label{thm:explicit:intersection}
  Let $x_0 \in \Omega_i$ and assume that $\partial \Omega_i \cap B_{2R}(x_0) = \emptyset$ for $R = r_0 \sqrt{t}$, where $r_0 > \max\{2, \sqrt{2d/3}\}$. Further, $x \in B_{R}(x_0)$, and $v_{n,\Omega_i}$, $r$ and $\theta$ are as described in \cref{sec:geom_notation}. If $d \geq 1$ and $r<r_0/2$, then we have that
  \begin{equation*}
    L^i_t f(x) = t^{\frac{d+1}{2}}\left(A(d,r_0,\theta) v_{n,\Omega_i}\sin\theta r{\e}^{-\sin^2\theta r^2} + B(x){\e}^{-r_0^2}\right),
  \end{equation*}
  where $A,B$ are real-valued functions. The function $B$ depends on $x$, but it is uniformly bounded by $|B(x)| \leq \frac{d+1}{4} r_0^{d} p|\mathbb{S}^{d-1}|$; and $A$ depends on $x$ only through $\theta$, and is bounded by
  \begin{equation*}
    \frac{p}{2}\pi^{d/2} \leq A(d,r_0,\theta)\leq p\pi^{d/2}.
  \end{equation*}
\end{theorem}

The following theorem is an extension of \cref{thm:explicit:intersection} to the case when the ball $B_R(x_0) \cap \partial \Omega_i \neq \emptyset$, which gives rise to an additional term in the expression of $L_t^if(x)$. We again refer to the schematic picture of \cref{fig:thm2} and comments in \cref{sec:geom_notation} for explanation of the coordinates $(r,\theta)$, function $v_{n,\Omega_i}$ and constant $v_{n,\partial \Omega_i}$.
\begin{theorem}\label{thm:explicit:boundary}
  Let $x_0 \in \Omega_i$ and assume that $\partial \Omega_i \cap B_{2R}(x_0)$ is part of a $d-1$ dimensional plane for $R = r_0 \sqrt{t}$, where $r_0 > \max\{2, \sqrt{2d/3}\}$. Further, $x \in B_{R}(x_0)$, and $v_{n,\Omega_i}$, $v_{n,\partial \Omega_i}$, $r$ and $\theta$ are as described in \cref{sec:geom_notation}. If $d \geq 1$ and $r<r_0/2$, then we have that
  \begin{align*}
    L^i_t f(x) & =\widehat{A}_1(x) t^{\frac{d+1}{2}} v_{n,\Omega_i} \sin\theta\, r {\e}^{-\sin^2\theta\, r^2} + \widehat{A}_2(x) t^{\frac{d+1}{2}} v_{n,\partial \Omega_i} {\e}^{-\sin^2\theta\, r^2} \\
               & \quad + B(x)t^{\frac{d+1}{2}} {\e}^{-r_0^2},
  \end{align*}
  for explicitly computable function $\widehat{A}_2$, and with explicitly computable bounds of function $\widehat{A}_1$. The function $B$ has the same bounds as in \cref{thm:explicit:intersection}.
\end{theorem}
\begin{remark}
  The function $\widehat{A}_1$ is bounded by
  \begin{equation*}
    \frac{p}{2\delta_0}\left({\e}^{-k_0^2}\logamma{\frac{d-1}{2}}{\delta_0^2-k_0^2}-{\e}^{-\delta_0^2}\frac{2(\delta_0^2-k_0^2)^{\frac{d-1}{2}}}{d-1}\right) \leq \widehat{A}_1 \leq p\Gamma\left(\frac{d-1}{2}\right)\sqrt{\pi}
  \end{equation*}
  and $\widehat{A}_2$ is given by
  \begin{equation*}
    \widehat{A}_2 =  \frac{p|\mathbb{S}^{d-2}|}{2}\left({\e}^{-\delta_0^2}\frac{(\delta_0^2 - k_0^2)^{(d-1)/2}}{d-1} +\frac{1}{2}{\e}^{-k_0^2}\gamma\left(\frac{d-1}{2},\delta_0^2-k_0^2\right)\right),
  \end{equation*}
  where $\Gamma(\cdot)$ is the \emph{Gamma function}, and $\gamma(\cdot,\cdot)$ is the \emph{incomplete lower Gamma function}, see \cref{A:gamma} for details.
  To define $k_0$ and $\delta_0$, we recall the geometric picture of \cref{sec:geom_notation}. Then $K$ is the projection of $(\hat x - \hat x_{\partial \Omega_i})$ to $n_{\partial \Omega_i}$, $k_0 = K/\sqrt{t}$, and $\delta_0 = \sqrt{r_0^2 - r^2\sin^2 \theta}$.
\end{remark}

\subsection{General manifolds}\label{sec:general}
In this section we will not assume the manifolds are flat, instead we assume that $\Omega$ is $(L,2R)$-regular, see~\cref{def:Lr-regular}. The type of singularity we deal with for a more general manifold will be a~\ref{it:type2}, and we will assume we are not too close to any boundary. Although we achieve a similar result as in \cref{thm:explicit:intersection}, the constants depend on the curvature of the manifold which in many practical applications would be unknown. This comes from the non-local nature of the $t$-regularized Laplacian $L_t$. It is however useful if we are in a situation where some a-priori information about the curvature is known, for instance if we are interested in exploring level sets of certain functions.

\begin{theorem}[General manifold]\label{thm:general}
  Let $x_0 \in \Omega_i$ and assume that $\partial \Omega_i \cap B_{2R}(x_0) = \emptyset$ for $R = r_0 \sqrt{t}$, where $r_0 > \max\{2, \sqrt{2d/3}\}$. Further, $x \in B_{R}(x_0)$, and $v_{n,\Omega_i}$, $r$ and $\theta$ are as described in \cref{sec:geom_notation}. If $d \geq 1$ and $r<r_0/2$, then we have that
  \begin{equation*}
    \begin{split}
      L_t^if(x) = t^{\frac{d+1}{2}} \widehat A(x)v_{n,\Omega_i}r\sin \theta {\e}^{-r^2 \sin^2 \theta}+ t^{\frac{d+2}{2}} \widetilde{C}_{L,r_0}(x)+ D(x){\e}^{-r_0^2}.
    \end{split}
  \end{equation*}
  In the above, $\widehat A > 0$ is a function such that
  \begin{equation*}
    |A(d,r_0,\theta) - \widehat A(x)| \leq (1+4Lr_0^2 t)A(d,r_0,\theta)
  \end{equation*}
  where $A(d,r_0,\theta)$ as in \cref{thm:explicit:intersection};  $\widetilde{C}_{L,r_0}$ is a function such that
  \begin{equation*}
    |\widetilde{C}_{L,r_0}(x)|\leq 2p\pi^{d/2}\left((8L r_0^4 + 16L^2 r_0^5 \sqrt{t})\, {\e}^{8L r_0^3 \sqrt{t} + 16L^2 r_0^4 t} + L r_0^2 (1+4L r_0^2 t)\right)
  \end{equation*}
  and $|D(x)| \leq \diam(\Omega)$.
\end{theorem}

The next lemma gives useful bounds on $L^i_tf(x)$ when $x$ is non-singular.
\begin{lemma}\label{lem:error}
  Given the conditions of \cref{thm:general} and the additional assumption that $x \in \Omega_i$, we have that
  \begin{equation*}
    L_t^i f (x)  = t^{\frac{d+2}{2}} E_{L,r_0}(x) + D(x){\e}^{-r_0^2},
  \end{equation*}
  where $E_{L,r_0}$ is a function satisfying
  \begin{equation*}
    |E_{L,r_0}(x)| \leq \widehat A(x) L r_0^2 + |\widetilde{C}_{L,r_0}(x)|.
  \end{equation*}
\end{lemma}

\begin{remark}
  The result in \cref{lem:error} can be used together with both \cref{thm:explicit:intersection} and \cref{thm:general} to analyze the behavior of the mapping $x \to L_tf(x)$ around intersections.
\end{remark}
The following corollary is a consequence of \cref{thm:general,lem:error}.
The geometry is as in \cref{sec:general_geom_notation}, projecting $x$ specifically to the tangent plane $T_{\Omega_1,x_0}$.
\begin{corollary}\label{cor:intersection}
  Let $x_0 \in  \Omega_1 \cap \Omega_2$ and assume that $\partial \Omega_i \cap B_{2R}(x_0) = \emptyset$ for $i \in \{1,2\}$ and $R = r_0 \sqrt{t}$, where $r_0 > \max\{2, \sqrt{2d/3}\}$. Assuming $R$ and $t$ satisfies the conditions of \cref{thm:general}, then for $x \in B_R(x_0)\cap \Omega_2$ such that $\|x - x_0\| = r \sqrt{t}$ for $r < r_0/2$, we have that
  \begin{equation*}
    \begin{split}
      L_tf(x) & =  t^{\frac{d+1}{2}}\widehat A(x)v_{n,\Omega_1}r\sin \theta_1 {\e}^{-r^2 \sin^2 \theta_1} \\
              & \qquad + t^{\frac{d+2}{2}} E_{L,r_0}(x) + 2D(x){\e}^{-r_0^2}
    \end{split}
  \end{equation*}
  where $E_{L,r_0}$ is a function satisfying
  \begin{equation*}
    |E_{L,r_0}(x)| \leq \widehat A(x) L r_0^2 + 2|\widetilde{C}_{L,r_0}(x)|.
  \end{equation*}
  In the above, $\theta_1$ and $v_{n,\Omega_1}$ are as in \cref{sec:general_geom_notation}, with $\Omega_i = \Omega_1$. Functions $\widehat A, \widetilde{C}_{L,r_0}$ and $D$ are as in \cref{thm:general}.
\end{corollary}

\subsection{Manifolds with noise}\label{sec:noise}
In earlier sections, we have assumed that the samples used to evaluate $L_{n,t}f(x)$ are taken directly from $\Omega$. However, in many real-world applications we can only expect that we access the points $X_1, X_2, \dots X_n$ approximately, for instance due to measurements uncertainties. In this section we propose a scenario where our results can have application also in this more general case. For instance, we could have the following situation:

\begin{enumerate}
    \item We first obtain a set of observations $x_1, x_2, \dots x_n$ of $X_1, X_2, \dots, X_n$ from $\Omega$.
    \item We cannot observe these points directly, but instead we are able to take multiple, noisy, measurements from the random variables $x_1 + \epsilon_1, x_2 + \epsilon_2, \dots x_n + \epsilon_n$, where each $\epsilon_i \sim \mathcal{N}(0,\sigma^2 I)$ \label{noisy}
\end{enumerate}

To model this is, we replace the operator
\begin{equation*}
  L_{n,t}f(x) = \frac{1}{n}\sum_{j=1}^n K_t(x,X_j)(f(x)-f(X_j)),
\end{equation*}
by
\begin{equation*}
  L_{n,t,\epsilon}f(x) = \frac{1}{n}\sum_{j=1}^n K_t(x,X_j+\epsilon_j)(f(x)-f(X_j+\epsilon_j)).
\end{equation*}
The following theorem shows that with our results, and with regards to the second source of randomness, the uncertainty in measurements, we can understand the expectation of $L_{n,t,\epsilon}$.
\begin{theorem}[Stochastic version]\label{thm:noisy}
  Let $L_{n,t,\epsilon}$ be as above, and the operator $\E_{\epsilon}[\,\cdot \,] = \E[\, \cdot \mid X_1,\dots,X_N]$ be expectation with regard to the random variables $(\epsilon_1,\dots, \epsilon_n)$. Then
  \begin{equation*}
    \E_{\epsilon}  L_{n,t,\epsilon}f(x) =  \frac{t^{N/2+1}}{(2\sigma^2 + t)^{N/2+1}} \frac{1}{n}\sum_{j=1}^n K_{2\sigma^2 + t}\left(x,X_j\right) (f(x)-f(X_j)).
  \end{equation*}
\end{theorem}

Thus, for $t' = 2\sigma^2+t$ and up to normalization, $\E_{\epsilon}L_{n,t,\epsilon}$ is the operator $L_{n,t'}$.

\section{Hypothesis test for singularities in flat manifolds}\label{sec:hypothesis}

In this section we will develop a hypothesis test for singularities in flat manifolds. We begin by assuming that $\Omega = \cup_i \Omega_i$ is a union of flat manifolds, and that we have a set of samples $\{X_1,\ldots,X_n\}$ from $p$. In this section we wish to consider the following hypothesis for a fixed point $x_0 \in \Omega$:
\begin{itemize}
  \item $H_0$: $B_{2}(x_0) \cap \Omega = B_{2}(x_0) \cap \Omega_i$ for some $i$ and $B_{2}(x_0) \cap \partial \Omega_i = \emptyset$.
  \item $H_1$: Hypothesis $H_0$ is false.
\end{itemize}

\begin{theorem} \label{thm:hypothesis}
  Let $\Omega = \cup \Omega_i \subset \R^N$ be a union of flat manifolds, and $X_1,\dots,X_n$ be i.i.d.~samples from $p$. Define the statistic
  \begin{equation*}
    T = \max_{m: X_m \in B_1(x_0)} |L_{n,t} f(X_m)|,
  \end{equation*}
  and for a fixed test-level $\alpha \in (0,1)$ the rejection region $\Theta := \{T > \delta\}$, where
  \begin{equation*}
    \delta = \sqrt{\frac{t}{\e(n-1)}\log \left (\frac{2n}{\alpha} \right )}.
  \end{equation*}
  Let $C_d = \frac{p(d+1)|\mathbb{S}^{d-1}|}{4} + 1$, where $d$ is the intrinsic dimension of $\Omega$. If
  \begin{equation*}
    t < \min \left \{\frac{1}{4}, \frac{3}{2d}, \frac{2}{\log \left (\frac{4C_d^2\e(n-1)}{\log \left (\frac{2n}{\alpha} \right ) }\right )}  \right \},
  \end{equation*}
  and we reject $H_0$ in the rejection region $\Theta$, then this is a test of level $\alpha$.
\end{theorem}

\begin{remark}
  Under the null hypothesis that $\Omega_i$ is $(L,R)$-regular, we can use the same test statistic $T$, but we change our rejection region by redefining $\delta$ using \cref{thm:general}. This $\delta$ will necessarily be larger and thus if we wish to detect singularities in more general manifolds, we need more samples to achieve the same power.
\end{remark}

As with all hypothesis tests we can only safely reject the null hypothesis, but why we reject it is not clear. Looking at \cref{thm:explicit:boundary} it could either be that we have a singularity of the types~\ref{it:type1}--\ref{it:type4} if we assume that all submanifolds are flat. However, the reason for rejection could also be that the manifold is not flat. What we will outline next is a way to estimate the power of the test in \cref{thm:hypothesis} in the case when the alternative hypothesis is that there is an intersection between two flat manifolds. Similar results can be obtained for a singularity of any of the types~\ref{it:type1}--\ref{it:type4}.

\begin{theorem} \label{thm:power}
  Assume $\Omega = \Omega_1 \cup \Omega_2$ and otherwise as in \cref{thm:hypothesis}.
  Let $H_1'$ be an alternative hypothesis that there is an intersection at $x_0$ between $\Omega_1$ and $\Omega_2$ where both are flat manifolds of intrinsic dimension $d$, where $x_0$ is a singular point of~\ref{it:type2} and $\partial \Omega_i \cap B_{2}(x_0) = \emptyset$, $i=1,2$. Let $T$ be the test statistic in \cref{thm:hypothesis} and $\delta$ be as in \cref{thm:hypothesis}. Then the probability of rejecting $H_0$ under $H_1'$ (\emph{power}) is bounded by
  \begin{equation*}
    \P(T > \delta \mid H_1') \geq 1-\alpha-\P(X_m \notin \mathcal{A})^n,
  \end{equation*}
  where $\mathcal{A} = (B_{1/2}(x_0) \setminus B_{1/4}(x_0)) \cap \Omega_2$,
  when $t$ satisfies the condition in \cref{thm:hypothesis},
  and $n$ is so large that
  \begin{equation} \label{eq:power}
    \begin{split}
      t \left ((d+1)\log(1/t) + \log \left (\frac{512}{p^2\pi^{d} |v_{n,\Omega_1}|^2} \right ) \right ) & \leq 2(1-\sin^2 \theta_1),
      \\
      t                                                                                            & \leq \frac{2}{\log \left (\frac{1024 {\e}}{p^2\pi^{d} |v_{n,\Omega_1}|^2 \sin^2 \theta_1}\right )}.
    \end{split}
  \end{equation}
\end{theorem}

\begin{remark}
  Asymptotically as $|v_{n,\Omega_1}|$ and $\sin \theta$ tends to zero, the number of samples needed to satisfy \cref{eq:power} is
  \begin{equation*}
    \frac{n}{\log n } = O\left (\frac{1}{|v_{n,\Omega_1}|^2 \sin^2 \theta_1}\right ).
  \end{equation*}
  As such we see that we drastically lose power if the angle $\theta_1$ is small, or if the normal vector $v_{n,\Omega_1}$ is small. This is intuitive, as the smaller $|v_{n,\Omega_1}|$ is, the less we see the intersection in the chosen direction $v$ and if it is zero, the Laplacian applied to $f$ will not see the intersection at all. On the other hand, if $\theta_1$ is small, the intersection becomes less distinct. When $\theta = 0$, the two manifolds are tangent, and from the perspective of the Laplacian (for small $t$), they are effectively the same manifold, as such there is no singularity.
  In a practical application we can search for a vector $v$ such that $T$ is maximized (on a subsample). Thus, the only real constraint is the actual angle of intersection between the manifolds.
\end{remark}

\section{Numerical Experiments}\label{sec:numerical}
\subsection{Hypothesis test for singularities}\label{sec:num:hypo}
In this section we will demonstrate the hypothesis test in \cref{thm:hypothesis} for a flat manifold. We will assume that we have a set of samples $(X_1,X_2,\ldots,X_n)$ distributed according to the uniform density on $\Omega$, and that we have a fixed point $x_0 \in \Omega$. We will assume that the null hypothesis $H_0$ is true, and that $x_0$ is not close to any boundary. Specifically we will assume that $N=3$ and that $\Omega$ is a piece of a hyperplane in $\R^3$.
We will then evaluate $L_{n,t}f(x)$ on $(X_1,X_2,\ldots,X_n)$ and calculate the test statistic $T$ as in \cref{thm:hypothesis}, with $t$ as in \cref{thm:power}. We will then calculate the rejection region $\Theta$ and see if $T$ is in $\Theta$. You will find the results in \cref{fig:hypothesis}.

To test the power of \cref{thm:power} we construct an intersection between two hyperplanes in $\R^3$ at different angles. The result is in \cref{fig:hypothesis}.

\begin{figure}[!ht]
  \begin{center}
    \begin{tabular}{|c||c||c|c|}
      \hline
      Samples          & Rejections under $H_0$ & Rejections under $H_1'$ ($\theta = \pi/4$) & $H_1'$ ($\theta = \pi/2$) \\
      \hline
      $2 \cdot 10^4$   & $0$                    & $0$                                        & $0.04$                    \\
      $3 \cdot 10^4$   & $0$                    & $0$                                        & $0.73$                    \\
      $4 \cdot 10^4$   & $0$                    & $0$                                        & $1$                       \\
      $5.5 \cdot 10^4$ & $0$                    & $0.2$                                      & $1$                       \\
      $6 \cdot 10^4$   & $0$                    & $0.28$                                     & $1$                       \\
      $6.5 \cdot 10^4$ & $0$                    & $0.57$                                     & $1$                       \\
      $7 \cdot 10^4$   & $0$                    & $0.81$                                     & $1$                       \\
      \hline
    \end{tabular}
  \end{center}
  \caption{Results of hypothesis test for singularities in flat manifolds, each experiment run 100 times.}
  \label{fig:hypothesis}
\end{figure}

\subsubsection{Neural networks} \label{sec:num:neural}
As mentioned in the introduction (see \cref{sec:neural}) we are interested in the zero set of a simple, non-trivial, neural network. Specifically, let us consider the zero set $\Omega_\delta$ of a neural network as in \cref{sec:neural}, where $k=3$ and the target function $f_{W^\ast}$ is chosen such that all three weights $w_i$ are the same, $a_1 = -a_2 = -a_3$ and $\delta = 10^{-16}$. The choice of $k>2$ ensures that the zero set $\Omega_\delta$ is large, as due to cancellations in the sum in \cref{eq:net_sets}, there should be many weights that produce the same function $f_{W^\ast}$. Note that the dimension of the weights $W$ is $2k=6$, in this example.

As described in \cref{sec:constraint}, the dataset we use is an approximation of
\begin{equation*}
  \widehat{\Omega}_\delta := \{W \in \mathbf{B}: |f_W(x_i) - f_{W^\ast}(x_i)| < \delta, \quad x_i \in \mathcal{D}  \},
\end{equation*}
where $\mathbf{B}$ is the box $[-10,10]^6$ and $\mathcal{D}$ is a dataset of $100$ points on the unit circle.
The method used to generate the dataset, as described in \cref{sec:constraint}, is called constraint propagation. It can be computationally intensive, but leads to a dataset of $12,753,597$ points $X$ with the following key properties:

\begin{enumerate}
\item Each point in $X$  is within a small distance ($0.005\sqrt{6}$) of a point in $\widehat{\Omega}_\delta$.
\item If we put a box around each point in $X$ with side-length $0.01$, then the union of these boxes will cover $\widehat{\Omega}_\delta$.
\end{enumerate}

A larger number of nodes ($k$) or a smaller side-length of the boxes leads to an increase in the computational cost.
Our particular choice of box-size ($0.01$) and $k=3$ were made to balance computational tractability, accuracy and the possibility of a complex enough zero set.

The zero set itself $\widehat{\Omega}_\delta$, by its construction and the piecewise linearity of the activation function, constitutes a union of flat (linear) manifolds, as highlighted in \cref{sec:flat}. 

In \cref{fig:neural} we have for illustration purposes projected a subsample of our $X$ of $\widehat{\Omega}_\delta$ to $\R^3$, using PCA, and this plot suggests that these manifolds are also intersecting. We wish to apply the hypothesis test in \cref{thm:hypothesis} to verify this.

We can see the result of applying the graph Laplacian evaluated on a subsample of $X$ to a linear function $f(x)=v \cdot x$ in \cref{fig:neural} ($v$ was chosen to maximize the absolute value of the operator $L_t f$). The dashed line is the rejection threshold from \cref{thm:hypothesis}, with parameter $t$ chosen according to \cref{thm:power}. The direction $v$ was chosen on an independently sampled subset of $X$ and the test performed on another independent subset of $30000$ samples. The conclusion is that we can reject the null hypothesis that $X$ is a union of flat manifolds that do not intersect at scale $R=1$ where $x_0 = 0$.

As the above hypothesis test rejects the null hypothesis, we can look closer at the behavior of $L_t f$ in the vicinity of $x_0$.
Note that, in the case where close to $x_0$, the manifold $\Omega = \cup_{i=1}^m \Omega_i$ is an intersection between flat manifolds but without their boundaries intersecting $B_2(x_0)$, see \cref{thm:explicit:intersection}. Then each expected operator $L_t^i$ will have the form
\begin{equation*}
  L_t^i f(x) = t^{\frac{d+1}{2}} \left ( A(d,r_0,\theta_i) v_{n,\Omega_i}\sin\theta_i r{\e}^{-\sin^2\theta_i r^2} + B(x){\e}^{-r_0^2} \right ).
\end{equation*}
If the vector $v$ found above, makes one of $|v_{n,\Omega_i}|$ large (recall the definition from \cref{sec:geom_notation}) and the others small, we would expect to see the first term in the above expression present in the full operator $L_t f$. Indeed, since $v \cdot x \approx \sin \theta_i r$ for $x$ close to $x_0$ the first term of this operator exhibits behavior similar to the function $y{\e}^{-y^2}$ with respect to $y=\sin \theta_i r$ for small enough $t$. By our choice of $v$, it is clear in the right panel of \cref{fig:neural} that such a term is present in the operator $L_t f$ if we compare with the ideal signal in \cref{fig:signal}. This is a strong indication that the zero set of the neural network close to $x_0$ contains an intersection of flat manifolds.

\begin{figure}[!ht]
  \centering
  \includegraphics[width=0.45\textwidth]{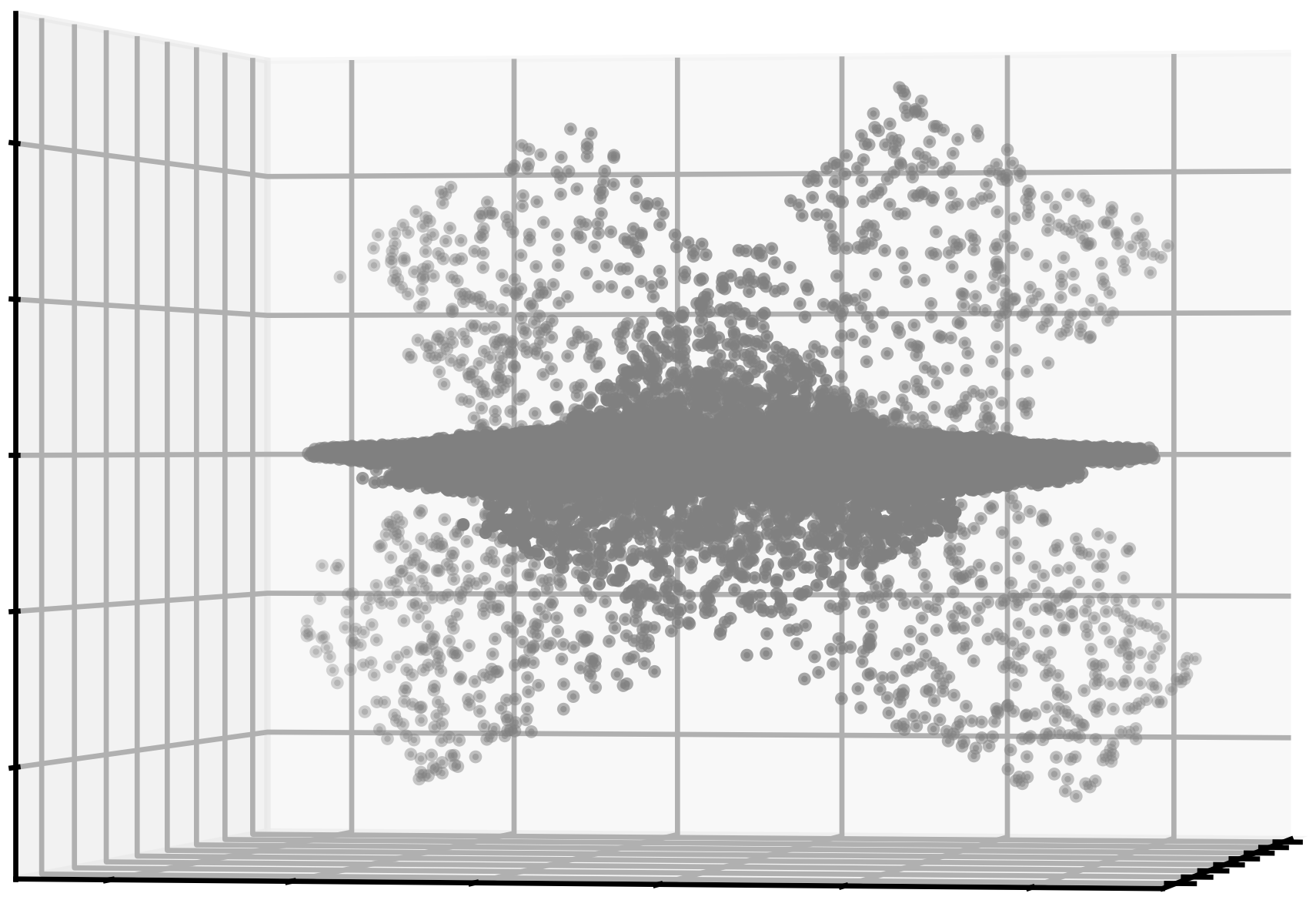}
  \includegraphics[width=0.45\textwidth]{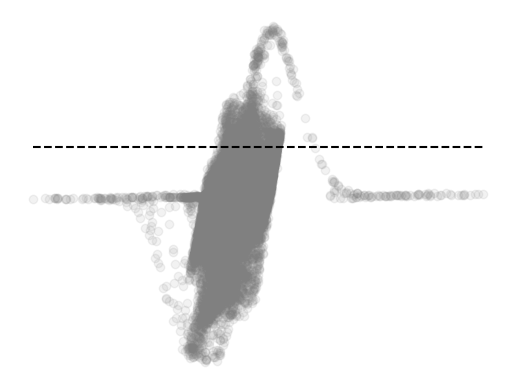}
  \caption{Two visualizations of neural network analysis: (left) PCA projection of a neural network's zero set from $\mathbb{R}^6$ to $\mathbb{R}^3$; (right) Graph Laplacian analysis showing $L_t f(x)$ on the $y$-axis and $f(x)$ on the $x$-axis, where $f(x) = v \cdot x$. The dashed line in the right plot represents the rejection region $\Theta$ specified in \cref{thm:hypothesis}.}
  \label{fig:neural}
\end{figure}

\subsection{Estimating singularities}\label{sec:num:sing}
In the following experiments we demonstrate how to estimate the location of an intersection and the intersecting angle $\theta$ of a union of two manifolds, given that we know that $\Omega = \Omega_1 \cup \Omega_2$, and $\Omega_1 \cap \Omega_2$ is non-empty and we have a~\ref{it:type2} singularity. Further, we assume that we have both a set of samples $X \subset \Omega$ distributed according to the associated density on $\Omega$,  and an additional set of points $Y$ from a curve $\Gamma \subset \Omega_i$, for some $i \in \{1,2\}$. The curve $\Gamma$ intersects $\Omega_1 \cap \Omega_2$, and we assume that no other singularity is very close, which is a situation like in \cref{fig:omega_flat} and \cref{fig:omega_curved}.

This particular setup could, for instance, be understood in the context of urban traffic flow analysis, where $\Omega_1$ and $\Omega_2$ represent two intersecting roads in a city, or some kind of interchange. In this scenario,
\begin{itemize}
    \item The set $X$ represents GPS data points from different vehicles traveling on the roads or interchange.
    \item The set  $Y \subset \Gamma$ represents samples from the path of a specific vehicle traveling through the intersection or interchange.
    \item The intersection $\Omega_1 \cap \Omega_2$ is the point of intersection between the roads, and the properties of $\Omega_1 \cap \Omega_2$ could be of interest for understanding traffic flow. 
\end{itemize}

\begin{figure}[!ht]
  \begin{tikzpicture}[scale=0.7]
    \begin{axis}[
        xmin = -3, xmax = 3,
        ymin = -0.5, ymax = 0.5,
        legend pos=north west]
      \addplot[
        domain = -3:3,
        samples = 200,
        smooth,
        thick,
        legend pos=north west,
      ] {x*exp(-x^2)};
      \legend{\(x {\e}^{-x^2}\)}
    \end{axis}
  \end{tikzpicture}
  \centering
  \caption{Graph of $y=x{\e}^{-x^2}$}
  \label{fig:signal}
\end{figure}
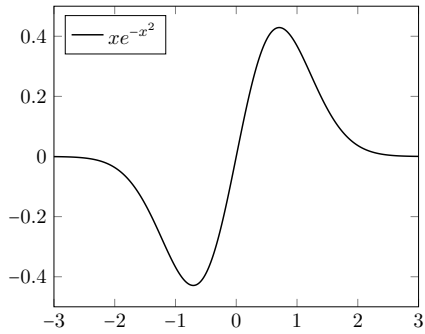

Now let $Y = \{y_1,\dots, y_m \} \subset \Gamma$ be as above. When we evaluate $x \mapsto L_{n,t}f(x)$ on $Y$, we get the values $P = \{p_1, \dots, p_m \ : \ p_i = L_{n,t}f(y_i), \ p_i \in \Gamma \}$. From our previous results, we know that these values, with enough samples, are close to 
\begin{equation}\label{eq:experiments}
  L_{t}f(y_i) = t^{\frac{d+1}{2}}\left(A(d,r_0,\theta_i) v_{n,\Omega_i}\sin\theta_i r_i{\e}^{-\sin^2\theta_i r_i^2}\right) + \text{Error}(y_i,r_0,L).
\end{equation}
The function depends on $r = \norm{y_i-x_0}$, and the angle  $\theta = \angle (y_i-x_0,\hat y_i-x_0)$, where $\hat y_i$ is the projection of $y_i$ onto either, in the flat case, a manifold $\Omega_i$ or, in the curved case, its tangent plane, as defined in \cref{thm:general}, for some point $x_0 \in \Omega_i$. The error term depends on $x$,$r_0$ and $L$. As before, the constant $L$ bounds the curvature of $\Omega$. The error term can be bounded using the results in \cref{sec:flat} and
\cref{sec:general}. By choosing a sufficiently small $t$, we can make $r_0$ arbitrarily large, causing the function $A(d,r_0,\theta)$ to approach $p\pi^{d/2}$. For our purposes in this section, we assume that $A(d,r_0,\theta) \approx p\pi^{d/2}$ and that $\text{Error}(x,r_0,L) \approx 0$ are good enough approximations.  With these assumptions,
\begin{equation}\label{eq:experiments:approx}
  L_{t}f(y_i) = t^{\frac{d+1}{2}}p\pi^{d/2}v_{n,\Omega_i}\sin\theta_i r_i{\e}^{-\sin^2\theta_i r_i^2}.
\end{equation}
\begin{remark}
  \cref{thm:general} uses $\widehat A(x)$ instead of $A(d,r_0,\theta)$, but $\widehat A(x)$ can be made arbitrarily close to $A(d,r_0,\theta)$ by choosing $t$ or $L$ small enough. How small we can make $L$, however, depends on how well-behaved the manifold is.
\end{remark}

We will use the functional form of equation \cref{eq:experiments:approx} and our two sets of samples to get estimators to the location and the angle of intersection of our manifold. We do not give explicit probabilistic bounds here, but assume $L_{n,t}f$ is sufficiently close to $L_{t}$.

\begin{remark}
    Let us note that a similar concentration bound as in \cref{thm:finite:sample} should be possible for this type of scenario, which would allow us to get explicit bounds for these experiments. That and not using the approximations above, would allow us to get explicit probabilistic intervals for all estimators.
\end{remark}

\subsubsection{Estimators for Intersection Location and Angle}
In this section we derive estimators for the location of the intersection and angle of intersection, based on \cref{eq:experiments:approx}.

Let $g(r,\theta) = v_{n,\Omega_i}r \sin \theta {\e}^{-r^2\sin^2 \theta}$. We observe that $g$ only depends on $x = r \sin \theta$ (up to the sign of $v_{n,\Omega_i}$), allowing us to write $g(r,\theta) = g(x)$, with $g(x)$ being a rescaled version of the function $h(x) = x{\e}^{-x^2}$. See \cref{fig:signal} for an illustration of $h$.

\sloppy The function $h$ attains its minimum and maximum values at $z_{1}=-\frac{1}{\sqrt{2}}$ and $z_2 = \frac{1}{\sqrt{2}}$ respectively. The intersection corresponds to the midpoint of these two points. We can estimate the point $s$ where $\Gamma$ intersects $\Omega_1 \cap \Omega_2$ as:
\begin{equation*}
  \hat{s} = \frac{\arg\max_{x_i}(P) + \arg\min_{x_i}(P)}{2}.
\end{equation*}

\sloppy To estimate the intersection angle $\hat \theta$ of $\theta$, we first need $\hat r_{\max}$, an estimate of the scaled distance from 0 that maximizes $g(r,\theta)$. This is
\begin{equation*}
\hat r_{\max} = \frac{\norm{\hat{s}-\arg\max_{x_i}P}}{\sqrt{t}}.
\end{equation*}
Since $\max_r g(r,\theta) = \frac{1}{\sqrt{2}}$, we have that $\hat r_{\max}\sin \theta \approx \frac{1}{\sqrt{2}}$. Solving for $\theta$, this leads to our estimator:
\begin{equation*}
  \hat{\theta} = \arcsin \left(\frac{1}{\sqrt{2}\hat r_{\max}}\right).
\end{equation*}

\subsubsection{Experimental setup}
We tested our estimation methods on hypersurfaces in $\mathbb{R}^3$ intersecting at angles given by $\theta \in \{\pi/2,\pi/4,\pi/8, \pi/16\}$, and with $t=10^{-3}$. For both flat and curved manifolds, we conducted 100 runs with random chosen functions of $f(x) = v\cdot x$, where $v$ is sampled from the uniform distribution on $\mathbb{S}^{2}$.

For each run:

\begin{itemize}
    \item We sampled $2\times 10^4$ points, from both $\Omega_1$ and $\Omega_2$, in a bounded region near the intersection
    \item We evaluated $L_{n,t}$ on $10^3$ uniformly spaced points  on $\Gamma$ to obtain our set $P$.
\end{itemize}

\begin{remark}
  We chose $\mathbb{R}^3$ mainly for visualization purposes. The intrinsic dimension of $\Omega$ remains low, regardless of the ambient dimension. However, choosing a function $f(x) = v\cdot x$ for $L_t^i$ to act on becomes more challenging in higher dimensions, as it becomes harder to orient $v$ such that $v_{n,\Omega_i}$ is large, especially when $v$ is chosen randomly, as in our experiments.
\end{remark}

In the following sections we present our results for each type of manifold and angle of intersection.
\subsubsection{Results for flat manifolds}
In this section, we apply our estimation methods to the case that $\Omega_1,\Omega_2$ are flat. Since the manifolds are flat, the angle $\theta_i = \theta$, where $\theta \in [0,\frac{\pi}{2}]$ is fixed.

In \cref{fig:omega_flat} we see our samples of $\Omega$ and $\Gamma$, and in \cref{fig:signal_flat} we see an example of the values we get in $P$. Finally, in \cref{fig:estimates_flat} we see how well this approach works in trying to learn both $\theta$ and $s$.

\begin{figure}[ht]
  \includegraphics[width=1\textwidth]{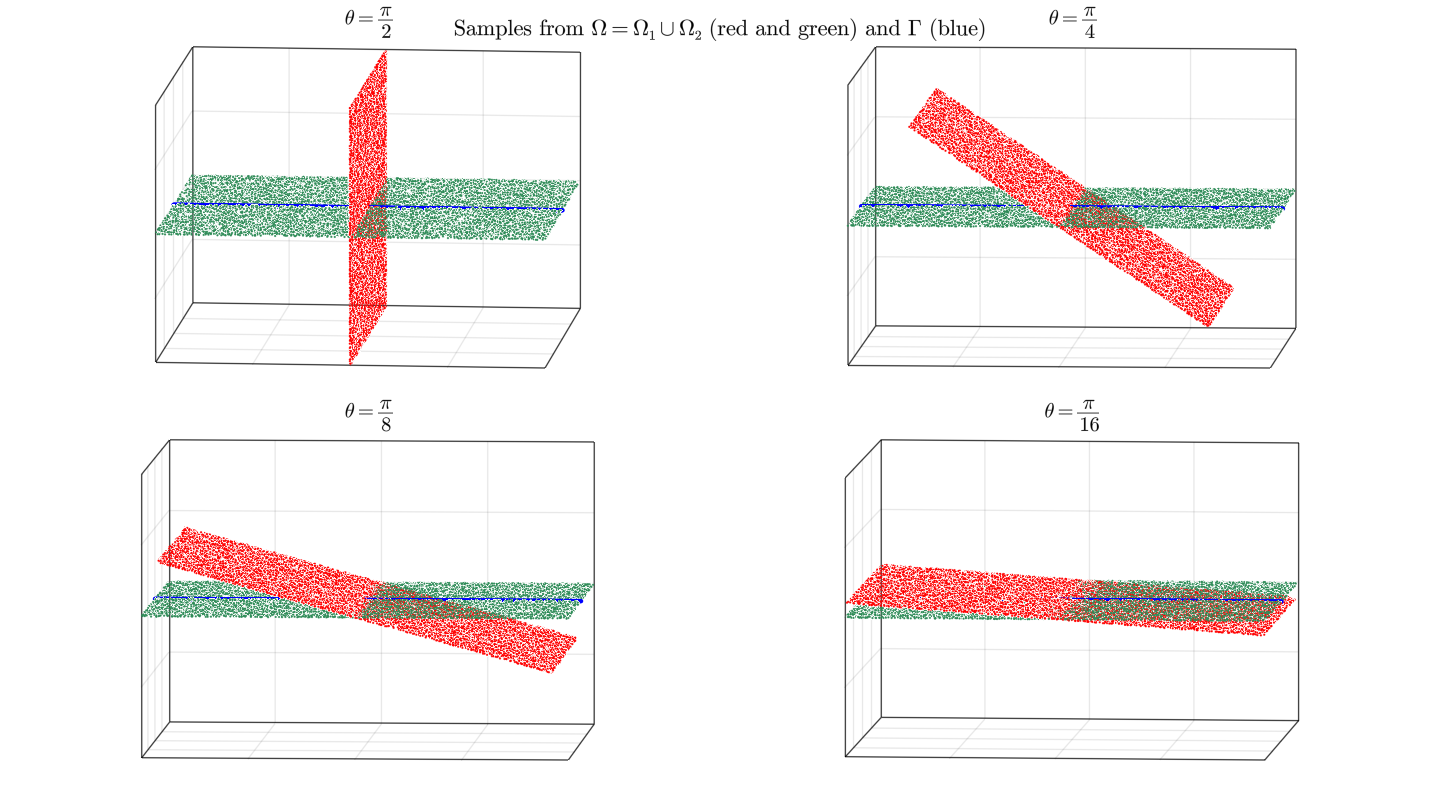}
  \centering
  \caption{Samples of $\Omega = \Omega_1 \cup \Omega_2$ and $\Gamma$ (blue), where $\Omega_1$ (green) and $\Omega_2$ (red) have no curvature.}
  \label{fig:omega_flat}
\end{figure}

\begin{figure}[ht]
  \centering
  \includegraphics[width=0.9\textwidth]{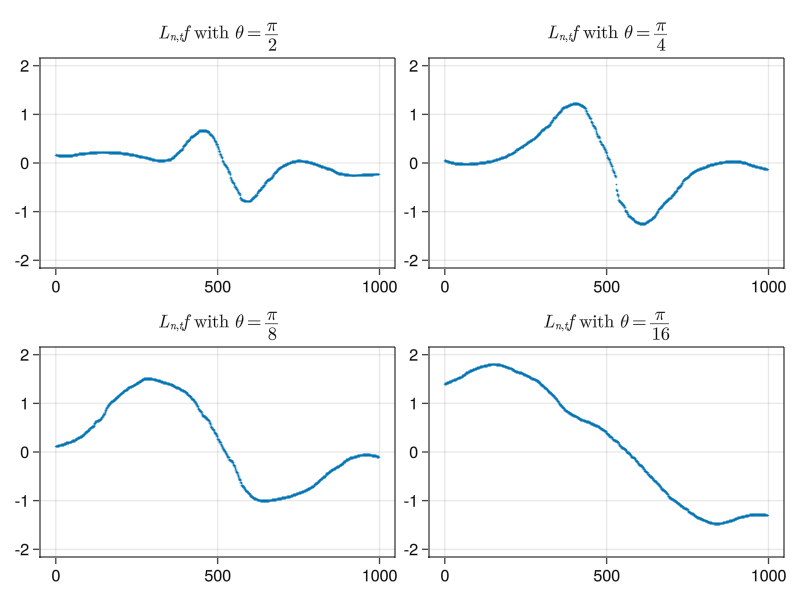}
  \centering
  \caption{$L_{n,t}f$ evaluated on $\Gamma$. Flat manifolds.}
  \label{fig:signal_flat}
\end{figure}

\begin{figure}[ht]
  \centering
  \begin{minipage}{0.45\textwidth}
    \centering
    \includegraphics[width=0.9\textwidth]{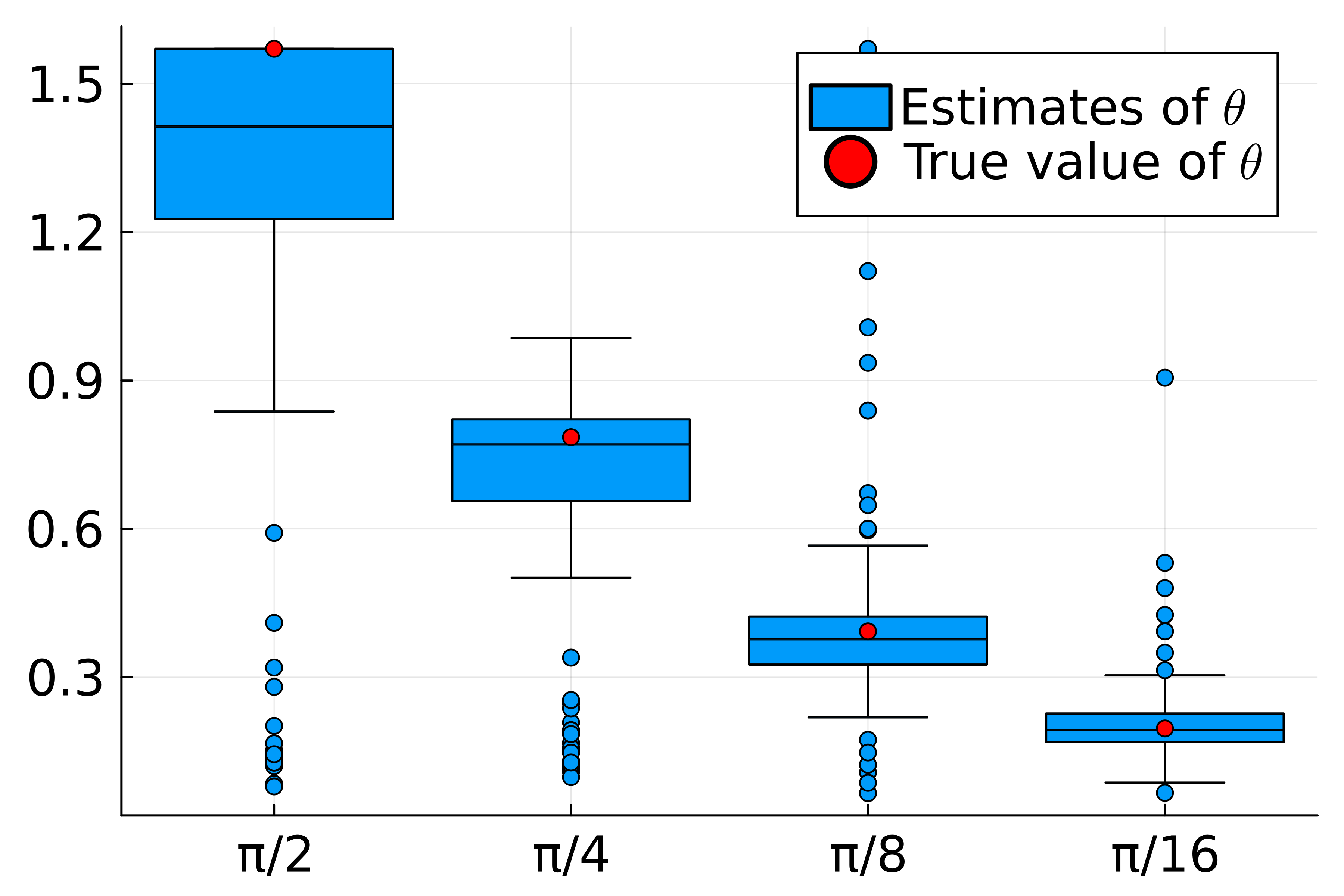}
  \end{minipage}
  \begin{minipage}{0.45\textwidth}
    \centering
    \includegraphics[width=0.9\textwidth]{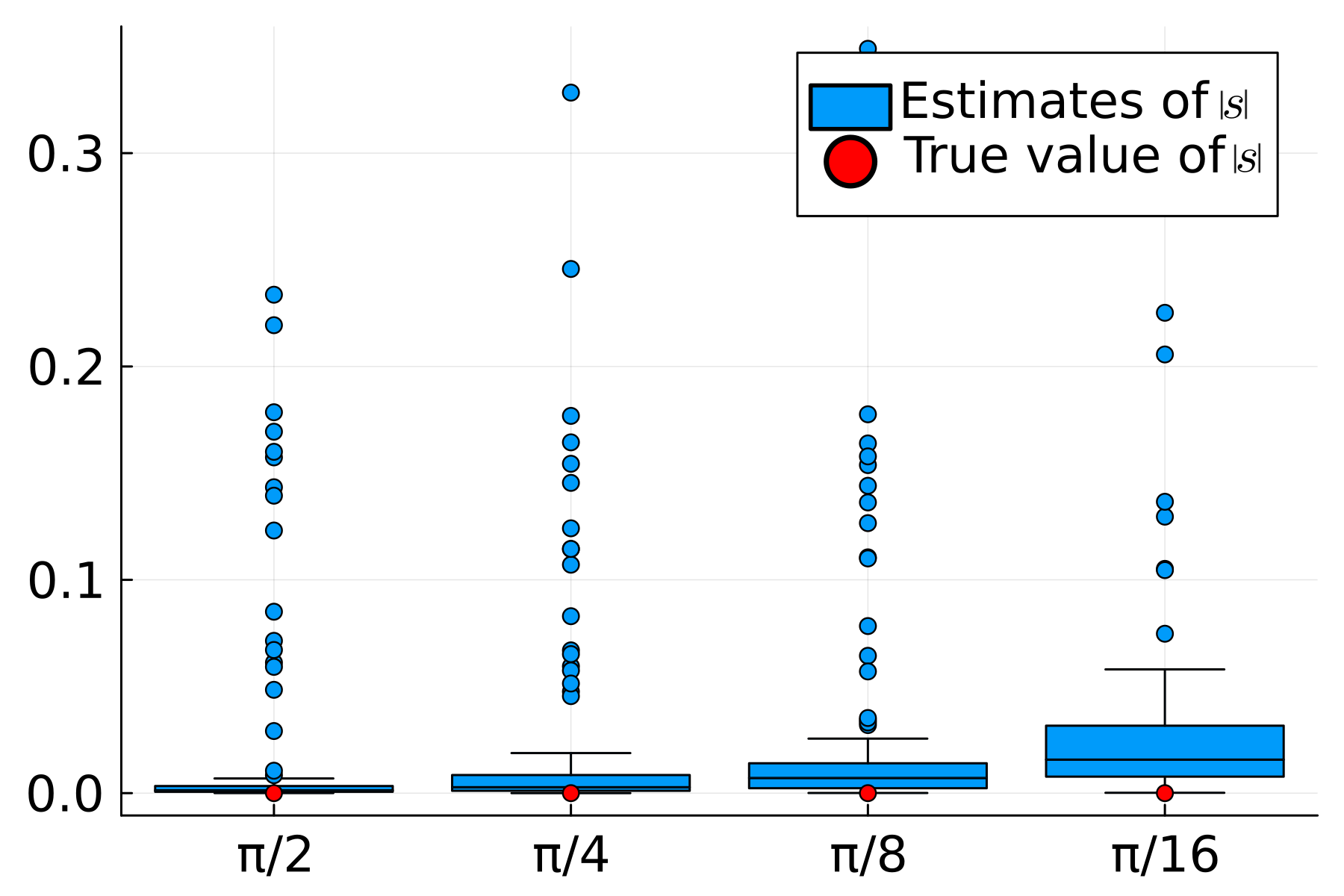}
  \end{minipage}
  \caption{Estimates of $\theta$ and $|s|$ on flat manifolds.}\label{fig:estimates_flat}
\end{figure}

\subsubsection{Results for curved manifolds}
Here we test these methods in the case that $\Omega = \Omega_1 \cup \Omega_2$ is $(L,2R)$-regular, with $L=0.5$ and $R$ having no upper bound. The setup is the same as what we see in \cref{fig:omega_curved}.

If we let $\theta$ denote the angle between the tangent spaces of $\Omega_1$ and $\Omega_2$ at $x_0$, then since there is curvature, we cannot expect $\theta_i = \theta$ for every $i=1,\dots ,n$, or even between any pair of them. However, we can still estimate the location of the intersection as before, and estimating $\theta$ in this way provides some information about the intersection, even if it is not as strong as in the case without curvature. The range of possible values for $\theta$, due to curvature, can be bounded by knowing the bounds of the curvature constant $L$.

In \cref{fig:omega_curved} we see our samples of $\Omega$ and $\Gamma$, in \cref{fig:signal_curved} we see an example of the values we get in $P$. Finally, in \cref{fig:estimates_curved} we see how well this approach works in trying to learn both $\theta$ and $s$.

\begin{figure}[ht]
  \includegraphics[width=1\textwidth]{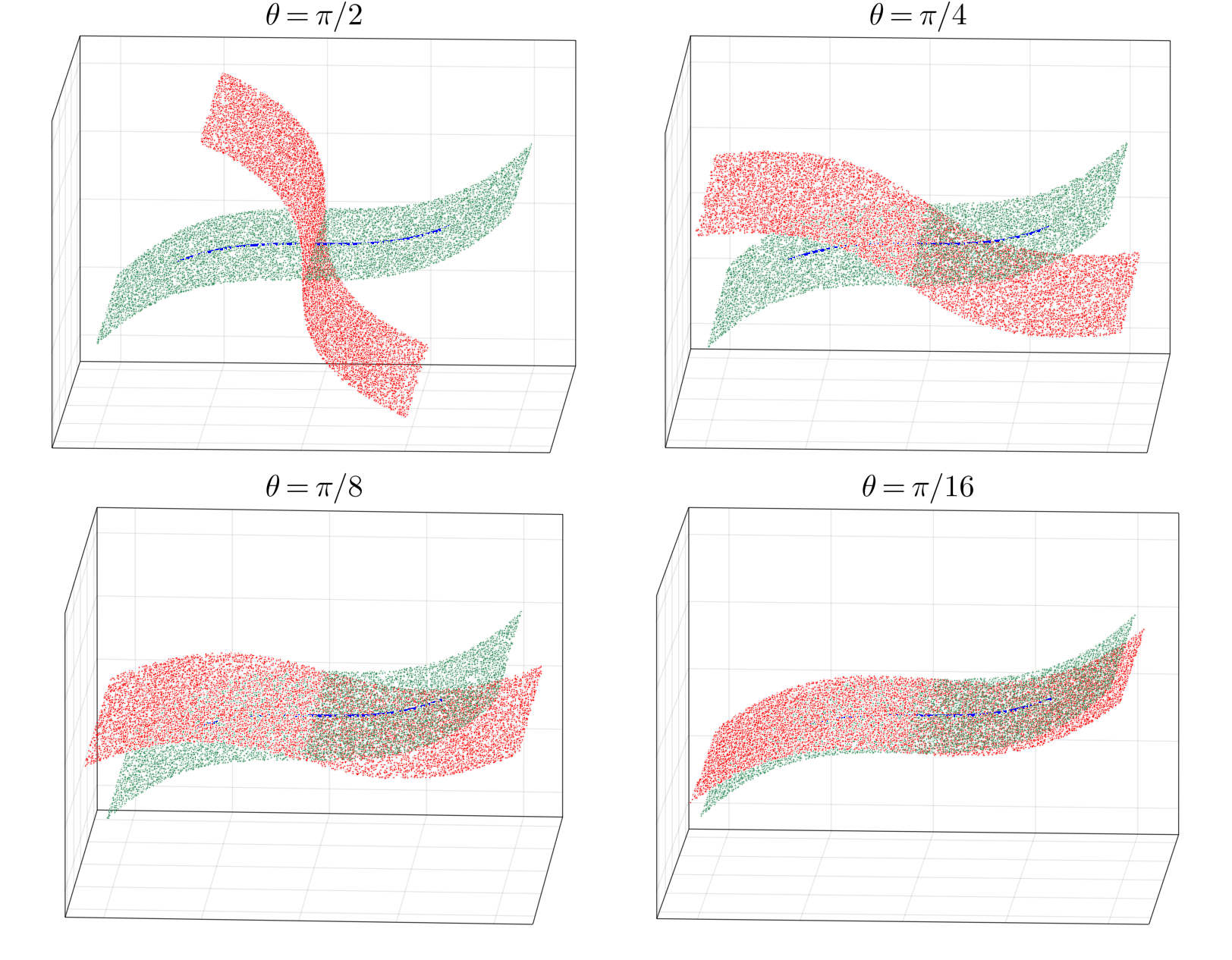}
  \centering
  \caption{Samples of $\Omega = \Omega_1 \cup \Omega_2$ and $\Gamma$ (blue), where $\Omega_1$ (green) and $\Omega_2$ (red) have curvature.}
  \label{fig:omega_curved}
\end{figure}

\begin{figure}[ht]
  \includegraphics[width=0.8\textwidth]{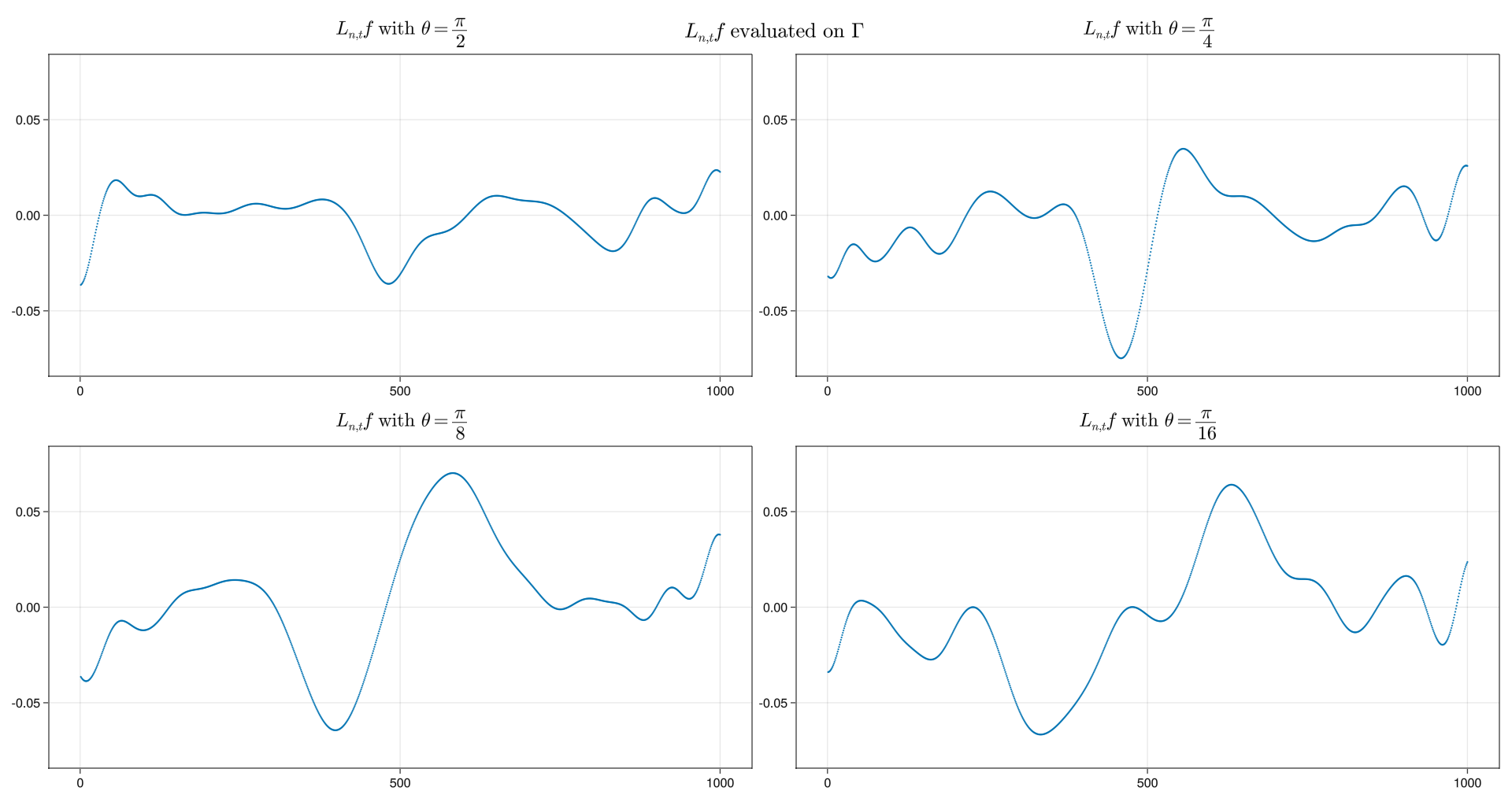}
  \centering
  \caption{$L_{n,t}f$ evaluated on $\Gamma$. Curved manifolds.}
  \label{fig:signal_curved}
\end{figure}

\begin{figure}[ht]
  \centering
  \begin{minipage}{0.45\textwidth}
    \centering
    \includegraphics[width=0.9\textwidth]{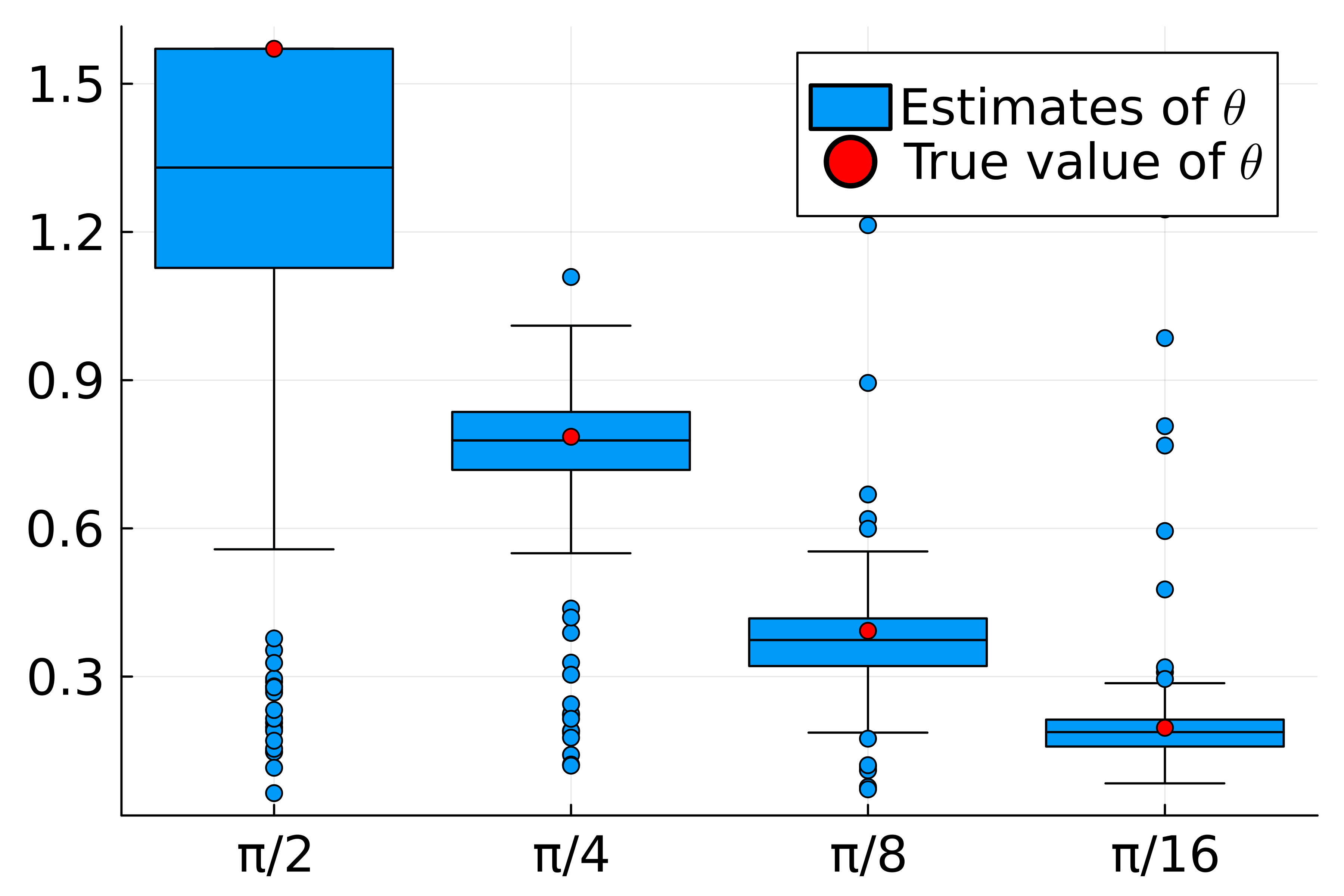}
  \end{minipage}\hfill
  \begin{minipage}{0.45\textwidth}
    \centering
    \includegraphics[width=0.9\textwidth]{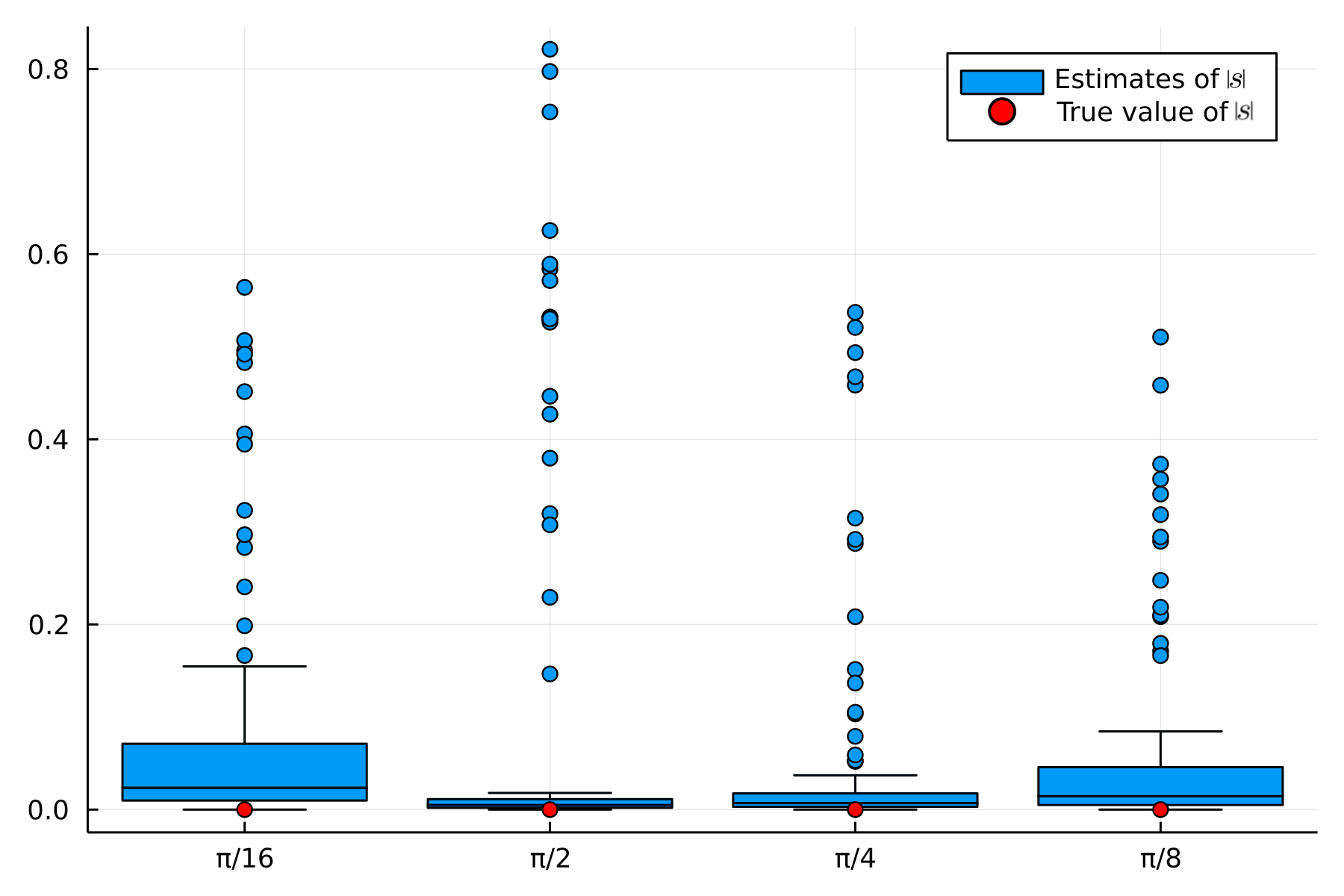}
  \end{minipage}
  \caption{Estimates of $\theta$ and $|s|$ on curved manifolds.}
  \label{fig:estimates_curved}
\end{figure}

\section{Final remarks}
In this paper we built upon the work of~\cite{BQWZ} and developed explicit versions of their asymptotic analysis of $x \to L_tf(x)$. Our results are the strongest and most useful in the case of flat manifolds, and the motivation to focus on this scenario comes partly from neural networks (see \cref{sec:neural}).

We used these explicit bounds to construct hypothesis tests with theoretical guarantees on the level of the test as well as the power under specific conditions. The numerical experiments in \cref{sec:num:hypo} suggest that the power can be close to $1$ with many samples as predicted by \cref{thm:power}. The power of these hypothesis tests can clearly be seen in the example of the zero set of a neural network (see \cref{sec:num:neural}) where we can safely reject the null hypothesis that the zero set is a single flat manifold.

While the bounds in \cref{thm:general} are weaker, our numerical experiments suggest that this approach can be useful for gaining geometric information about the union of more general manifolds $\Omega = \cup_i \Omega_i$. In~\cite{BQWZ}, the authors mainly considered sets $\Omega = \cup_{i=1}^n \Omega_i$ with $n \leq 2$. Our approach of splitting $L_t$ into components $L_t^i$ makes it straight forward to directly apply our theorems for $n \geq 2$, allowing us to consider a wider range of singularities. For example, we can extend the framework to examine points that are both of~\ref{it:type1} and~\ref{it:type2}, or to study intersections of more than two manifolds.

In our numerical experiments in \cref{sec:num:sing}, we assumed that $\Omega = \cup_{i=1}^2 \Omega_i$ and had access to samples of an intersecting curve $\Gamma$, which allowed us to estimate geometric properties near intersections. Future work could involve extending our framework to other types of singularities and developing similar tests and estimators. It would also be interesting to explore methods that do not rely on access to such curves.

Similar theorems can be proven for other kernels besides the Gaussian one, as many ideas used in our proofs are not specific to the Gaussian case, but rather rely mainly on symmetries of $K_t$. Investigating the use of other kernels and comparing their performance in different scenarios is a potential direction for future research.

\section*{Acknowledgments}
The first author was supported by the Wallenberg AI, Autonomous
Systems and Software Program (WASP) funded by the Knut and Alice Wallenberg Foundation.
The second author was supported by the Swedish Research Council grant dnr: 2019--04098.
The authors would like to thank the anonymous reviewers for their valuable comments and suggestions which helped to improve the quality of the paper.

\bibliographystyle{ieeetr}
\bibliography{bibliography}

\appendix

\section{Preliminaries on manifolds and explicit bounds on Gamma functions}\label{sec:preliminaries}

\subsection{Integration on $\Omega$}\label{sec:integration}
We will integrate scalar-valued functions, $f: \Omega \to \R$, over $\Omega$. When formulating integration of scalar-valued functions over submanifolds of $\R^N$, we follow the approach in~\cite{munkres2018analysis}. Because we need some preliminary results concerning integration on $\Omega$, we make some important definitions explicit.

First, let $x_1,\ldots,x_k$ be vectors in $\R^N$ for $k \leq N$. If $I = (i_1, i_2, \dots, i_k)$ is a $k$-tuple of integers such that $i_1 \leq i_2 \leq \cdots \leq i_k$, define $X_I \in \R^{k\times k}$ as the $k\times k$ matrix containing only rows $i_1, \dots,i_k$ of the matrix $X = (x_1,\ldots,x_k)$.
Now we can define the \emph{volume function} $V: \R^{N \times k} \to \R$, by $V(X) = \sqrt{\det^2(X^{t}X)} = \left[ \sum_I \det^2 X_I \right]^{1/2}$, where the $I$'s span over $k$-tuples as above, see~\cite[Theorem 21.4]{munkres2018analysis}.

In general, given a coordinate chart $\alpha: U \to W$, where $U \subset \R^d$, $W \subset \Omega_i \subset \mathbb{R}^N$ are open subsets, and $D\alpha$ is the Jacobian of $\alpha$, we can express integration over $W$ as
\begin{equation*}
  \int_W f \dV = \int_U f\circ \alpha \,  \mathrm{V} (D\alpha).
\end{equation*}

In the coming proofs, when integrating around a point $x\in \interior \Omega_i$, we will change coordinates to the standard basis in $T_{\Omega_i,x} \simeq \R^d$. With this we mean that we can find open sets $W \subset \Omega_i$ around $x$ such that the projection map $\pi : W \to B \subset x + T_{\Omega_i,x}$ is a diffeomorphism, where $x + T_{\Omega_i,x} \coloneqq \Set{x+y}{y \in T_{\Omega_i,x}}$. To integrate over $T_{\Omega_i,x}$ we use the map $\pi^{-1}$ precomposed with an inclusion map.

More specifically and without loss of generality, by translation and an orthonormal coordinate change, we can assume that $T_{\Omega_i,x} =\R^{d}\times \{0\}^{N-d}$.  In this coordinate system we can write
\begin{equation}\label{eq:chart}
  \alpha: U \xrightarrow{i} x + T_{\Omega_i,x} \xrightarrow{\pi^{-1}} W \subset \Omega_i,
\end{equation}
where $i$ is the natural inclusion map and $U$ an open subset in $\R^d$.

\subsection{Important bounds}

The following bounds will be used later in our proofs: First, let $T_{\Omega,x},U,W,\pi$ be as in \cref{sec:integration}. Then for any $y \in W$,
\begin{equation}\label{eq:bound1}
  \norm{y-\pi(y)} \leq O(\norm{x-\pi(y)}^2).
\end{equation}
This follows since $\Omega_i$ is smooth and the tangent space represents a first order approximation of $\Omega_i$ around $x$.

To formulate the second bound, we need the lemma below.
\begin{lemma}
  Let $U,W,x,y,\Omega_i,\pi,i,\alpha$ be as in \cref{sec:integration}. Then the following holds for the volume function $V$:
  \begin{equation*}
    V(D\alpha(y)) = 1 + O(\norm{x-\pi(y)}^2).
  \end{equation*}
\end{lemma}
\begin{proof}
  Since $\alpha = \pi^{-1} \circ i$, and the tangent space is a first order approximation of $\Omega$, we can parametrize the $W$ by $\alpha(u) = (u,g(u))$.
  It is then easy to see that for $i = 1,\ldots,d$ we have
  \begin{equation*}
    \partial_i \alpha(y) = (e_i,\partial_i g(u)),
  \end{equation*}
  where $\partial_i g(0) = 0$ and $\norm{\partial_i g(u)} = O(\norm{u})$.
  Now
  \begin{equation*}
    \det D\alpha_I =
    \begin{cases}
      1           & \text{if $I=(1,2\dots,d)$} \\
      O(\norm{u}) & \text{otherwise.}
    \end{cases}
    .
  \end{equation*}
  If we Taylor expand $x\to \sqrt{x}$, we get
  \begin{equation*}
    V(D\alpha) = \left(\sum_I (\det D\alpha_I)^2  \right)^{1/2} = 1 + O(\norm{u}^2),
  \end{equation*}
  and by applying the above on $(u,0) = x-\pi(y)$ we are finished.
\end{proof}
Further, since we have a finite union $\Omega = \cup_i \Omega_i$, $\Omega_i$ is compact, \cref{eq:bound1}, and the previous lemma implies that we can find a uniform bound $L$ such that for all tuples $(U,W,x,y,\pi, \Omega_i)$
\begin{equation*}
  \norm{y-\pi(y)} \leq L \norm{x -\pi(y)}^2
\end{equation*}
and
\begin{equation*}
  |V(D\alpha)-1| \leq L\norm{x-\pi(y)}^2
\end{equation*}
holds.

\subsection{Gamma functions} \label{A:gamma}
In the proofs of several of our results we will need to handle the \emph{Gamma function} $\Gamma(\cdot)$, and both the \emph{lower} and  \emph{upper incomplete gamma functions}, $\gamma(\cdot,\cdot)$ and $\Gamma(\cdot,\cdot)$ respectively. These are well-known and  are defined by the equations
\begin{equation*}
  \Gamma(a) = \int_0^\infty t^{a-1}{\e}^{-t}\dt,
  \quad
  \gamma(a,x) = \int_0^x t^{a-1}{\e}^{-t}\dt,
  \quad
  \Gamma(a,x) =  \int_x^\infty t^{a-1}{\e}^{-t}\dt.
\end{equation*}
In this paper both $a$ and $x$ are non-negative real numbers.

We will need the following bounds:
First, if $a\geq 1$, then $t^{a-1} \geq x^{a-1}$ and
\begin{equation}\label{eq:upper_gamma_bound2}
  \Gamma(a,x) \geq x^{a-1}\int_x^\infty {\e}^{-t} \dt = x^{a-1}{\e}^{-x}.
\end{equation}
Secondly, if ${\e}^x > 2^a$, then by~\cite[Theorem 4.4.3]{gabcke1979neue},
\begin{equation}\label{eq:upper_gamma_bound}
  \Gamma(a,x) \leq a x^{a-1}{\e}^{-x}.
\end{equation}
Finally, we need the lower bound
\begin{equation}\label{eq:lower_gamma_bound}
  \gamma(a,a) \geq \frac{1}{2}\Gamma(a).
\end{equation}
That this holds can be seen by viewing $\gamma(a,x)$ as an unnormalized version of the cumulative distribution function of the Gamma distribution, for which it is well-known that the median $\nu$ is less than $a$,~\cite{chen1986bounds}.

\section{Proof of \cref{thm:explicit:intersection,thm:explicit:boundary}}\label{sec:proof:explicit}

\subsection{Proof of \cref{thm:explicit:intersection}}

Since $x\to L_t^if(x)$ is translation and rotation invariant, we can without loss of generality assume that $\Omega_i$ oriented in $\R^N$ in such a way which makes it a subset of $\R^d\times \{0\}^{N-d}$.

We want to evaluate
\begin{equation*}
  L^i_t f(x) = \int_{\Omega_i} K_t(x,y) (f( x) - f(y))p \dy.
\end{equation*}
We begin by splitting the integral above into
\begin{align}
  \notag \int_{\Omega_i} K_t(x,y) (f( x) - f(y))p \dy & = \int_{B_{R}( x) \cap \Omega_i} K_t(x,y) (f( x) - f(y))p \dy            \\
  \notag                                              & \quad + \int_{\Omega_i \setminus B_{R}( x)} K_t(x,y) (f( x) - f(y))p \dy \\
  \label{eq:thm1_I1I2}                                & = I_1 + I_2.
\end{align}
For estimating $I_2$, by translation invariance we can without loss of generality assume that $x=0$. Now we make a change of variables and rescale $y$, which allows us to say that
\begin{equation*}
  \begin{split}
    |I_2| & = \left |\int_{\Omega_i \setminus B_{R}(0)} K_t(0,y) (f(0) - f(y)) p\dy \right |
    \\
          & = \left |\int_{\Omega_i \setminus B_{r_0 \sqrt{t}}(0)} {\e}^{-\|y\|^2/t} v \cdot (-y)p\dy \right |                                       \\
          & = \left |\int_{\left( \frac{1}{\sqrt{t}}\Omega_i \right) \setminus B_{r_0}(0)} {\e}^{-\|y\|^2} v \cdot (-y\sqrt{t}) t^{d/2}p\dy \right | \\
          & \leq t^{\frac{d+1}{2}}\int_{\R^d \setminus B_{r_0}} {\e}^{-\|y\|^2}  \|y\|p \dy.
  \end{split}
\end{equation*}
Now, by first changing to spherical coordinates and integrating out the angular parts, we deduce that
\begin{equation} \label{eq:thm1_I_2}
  \left| I_2 \right| \leq  t^{\frac{d+1}{2}}\left|\mathbb{S}^{d-1}\right| p\int_{r_0}^\infty {\e}^{-s^2}s^{d}ds = \frac{1}{2} pt^{\frac{d+1}{2}}\left|\mathbb{S}^{d-1}\right|\Gamma\left(\frac{d+1}{2},r_0^2\right).
\end{equation}
To finalize the bound of $I_2$, we note that by assumption $r_0 > \max\{2, \sqrt{2d/3}\}$, i.e., $r_0^2 > \max\{4, 2d/3\}$, which implies ${\e}^{r_0^2} > 2^{(d+1)/2}$ for $d \geq 1$, and we can use \cref{eq:upper_gamma_bound,eq:thm1_I_2} to conclude
\begin{equation} \label{eq:thm1_I_2_final}
  |I_2| \leq B(x)t^{\frac{d+1}{2}}{\e}^{-r_0^2},
\end{equation}
where $B(x)$ is some function such that
\begin{equation*}
  B(x) \leq \frac{d+1}{4}
  r_0^{d}p|\mathbb{S}^{d-1}|.
\end{equation*}

To bound $I_1$, we use the following simple geometric fact:
\begin{equation*}
  \|x - y\|^2 = \|\hat x - y\|^2 + \|\hat x-x\|^2 = \|\hat x - y\|^2 + \sin^2\theta r^2 t,
\end{equation*}
which implies that
\begin{equation*}
  {\e}^{-\| x - y\|^2/t} = {\e}^{-\sin^2\theta  r^2} {\e}^{-\| \hat x - y\|^2/t}.
\end{equation*}
From the above we can conclude
\begin{align}
  \notag I_1             & = {\e}^{-\sin^2\theta  r^2} \int_{B_{R}( x) \cap \Omega_i} {\e}^{-\|\hat x - y\|^2/t} v \cdot ( x-y) p\dy \\
  \notag                 &
  \begin{multlined}
    ={\e}^{-\sin^2\theta  r^2} \bigg( \int_{B_{R}( x) \cap \Omega_i} {\e}^{-\|\hat x - y\|^2/t} v \cdot (x- \hat x) p\dy \\
    + \int_{B_{R}( x) \cap \Omega_i} {\e}^{-\|\hat x - y\|^2/t} v \cdot (\hat x-y) p\dy \bigg)
  \end{multlined}
  \\
  \label{eq:thm1_II_III} & ={\e}^{-r^2\sin^2\theta } (II + III).
\end{align}

It is easier to integrate over a ball centered around $\hat x$, and to this end we define $\delta \geq 0$ by
\begin{equation}\label{eq:delta}
  \delta = \sqrt{R^2 - tr^2\sin^2\theta  }.
\end{equation}
Then since $\hat x$ is the orthogonal projection of $x$, we have that $B_R( x) \cap \Omega_i = B_\delta(\hat x) \cap \Omega_i$.

Let us focus on $II$: We use the \cref{eq:delta} and change to spherical coordinates, which yields
\begin{align}
  \notag II     & = v\cdot(x-\hat{x})t^{d/2}\int_{B_{\delta/\sqrt{t}}(\hat x)\cap \Omega_i}{\e}^{-\|\hat x-y\|^2}p\dy.                                                            \\
  \notag        & = v\cdot(x-\hat{x})t^{d/2}|\mathbb{S}^{d-1}| p\int_0^{\delta/\sqrt{t}}{\e}^{-s^2}s^{d-1}ds  = \frac{1}{2}v\cdot(x-\hat{x})t^{d/2}|\mathbb{S}^{d-1}|p\gamma(d/2,\delta^2/t) \\
  \label{eq:II} & = \frac{1}{2} v\cdot\frac{x-\hat{x}}{\|x-\hat{x}\|}t^{\frac{d+1}{2}}r\sin\theta|\mathbb{S}^{d-1}|p\gamma(d/2,\delta^2/t).
\end{align}
To estimate the RHS of \cref{eq:II} we will bound the $\gamma$-function from above and below:
Using $r_0^2 > \frac{2d}{3}$, $r<r_0/2$ and the definition of $\delta$, we get
\begin{equation*}
  \frac{d}{2}< \frac{3r_0^2}{4} < r_0^2-\sin^2 \theta r^2=\frac{\delta^2}{t}.
\end{equation*}
By \cref{eq:lower_gamma_bound} we now see that
\begin{equation}\label{eq:thm1_gamma_lower}
  \frac{1}{2}\Gamma(d/2) \leq\gamma(d/2,d/2)\leq\gamma(d/2,\delta^2/t).
\end{equation}
For the upper bound, we simply use
\begin{equation}\label{eq:thm1_gamma_upper}
  \gamma(d/2,\delta^2/t) \leq \Gamma(d/2).
\end{equation}
Now \cref{eq:II,eq:thm1_gamma_lower,eq:thm1_gamma_upper} together with $|\mathbb{S}^{d-1}| = \frac{2\pi^{d/2}}{\Gamma(d/2)}$ finally gives
\begin{equation*}
  II = A(d,r_0,\theta)v_{n,\Omega_i}t^{\frac{d+1}{2}}r\sin \theta,
\end{equation*}
where
\begin{equation}\label{eq:A}
  \frac{1}{2}p\pi^{d/2} \leq A(d,r_0,\theta) \leq p\pi^{d/2}.
\end{equation}

Finally, $III = 0$. This follows from that $B_R( x) \cap \partial \Omega_i = \emptyset$, the rotational symmetry of $K$, and the fact that the linear function is odd.
Collecting \cref{eq:thm1_I1I2,eq:thm1_I_2_final,eq:thm1_II_III,eq:II} we get
\begin{equation*}
  L_t^i f(x) = t^{\frac{d+1}{2}}\left(A(d,r_0,\theta) v_{n,\Omega_i}\sin\theta r{\e}^{-\sin^2\theta r^2} + B(x){\e}^{-r_0^2}\right).
\end{equation*}
{\flushright \qed}

\subsection{Proof of \cref{thm:explicit:boundary}}\label{sec:proof:boundary}
We will follow the proof of \cref{thm:explicit:intersection} and modify where needed.
Let $I_2,II$ and $III$ be defined as in \cref{eq:thm1_I1I2,eq:thm1_II_III}. Then, since $I_2$ is bounded like in \cref{eq:thm1_I_2_final}, we only need to find bounds for $II$ and $III$.

Let $\delta$ be defined as in \cref{eq:delta} and define $\delta_0 = \delta/\sqrt{t}$.
Recall also the fact that $B_R(x)\cap \Omega_i = B_\delta(\hat x) \cap \Omega_i$. Now  the difference in bounding $II$ and $III$ to the proof of \cref{thm:explicit:intersection} is that $B_\delta(\hat x) \cap \partial \Omega_i$ is nonempty.
Since, by assumption, $\partial \Omega_i$ is part of a $d-1$-dimensional flat space, $B_\delta(\hat x) \cap \Omega_i$ is a $d$-dimensional ball, but missing a spherical cap.

We now use cylindrical coordinates $(h,\varrho,\varphi)$ to describe the domain \newline $B_{\delta/\sqrt{t}}(\hat x)\cap \Omega_i$. In these new coordinates we are centered around $\hat x$, and $(\varrho,\varphi)$ are coordinates for a $d-1$-dimensional ball tangential to $\partial \Omega$, while the perpendicular coordinate $h$ is oriented along the outwards normal of $\partial \Omega_i$.
Let us denote this unit normal by $n_{\partial \Omega}$, and the projection of $\hat x$ to $\partial \Omega$ by $\hat x_{\partial \Omega}$. We now set $K = (\hat x - \hat x_{\partial \Omega}) \cdot n_{\partial \Omega} = \sqrt{t}k_0$, where $-\delta_0 \leq k_0 \leq \delta_0$.

Then, with $III$ defined in \cref{eq:thm1_II_III} we get
\begin{equation*}
  III=\int_{-\delta}^{K}\int_0^{\sqrt{\delta^2 - h^2}}\int_{\mathbb{S}^{d-2}}K_t(\hat x,y)v \cdot (\hat x -y) \varrho^{d-2}\dvar{\varphi}\dvar{\varrho}\dvar{h}.
\end{equation*}
We split $v$ into a normal component $v_n = (v\cdot n_{\partial \Omega}) n_{\partial \Omega}$ and a component $v_T = v-v_n$ which is tangential to the boundary $\partial \Omega$.
Then, since the function $y \to v_T \cdot (\hat x -y)$ is odd as a function centered around $\hat x$, and the domain of integration is symmetric around $\hat x$, we know that the tangential component of $III$ satisfies
\begin{equation*}
  III_T:=\int_{-\delta}^{K}\int_0^{\sqrt{\delta^2 - h^2}}\int_{\mathbb{S}^{d-2}}K_t(\hat x,y)v_T \cdot (\hat x -y) \varrho^{d-2}\dvar{\varphi}\dvar{\varrho}\dvar{h} = 0.
\end{equation*}
By definition of  $v_{n,\partial \Omega}$, we have that $v_n \cdot (\hat x - y) = v_{n,\partial \Omega}(n_{\partial \Omega} \cdot (\hat x -y)) = v_{n,\partial \Omega}h$, which implies that
\begin{align*}
  III
   & =
  v_{n,\partial \Omega}\int_{-\delta}^{K}\int_0^{\sqrt{\delta^2 - h^2}}\int_{\mathbb{S}^{d-2}}K_t(\hat x,y)h \varrho^{d-2}\dvar{\varphi}\dvar{\varrho}\dvar{h}
  \\
   & = v_{n,\partial \Omega}\int_{-\delta}^{K}\int_0^{\sqrt{\delta^2 - h^2}}\int_{\mathbb{S}^{d-2}}{\e}^{-h^2/t - \varrho^2/t}h \varrho^{d-2}\dvar{\varphi}\dvar{\varrho}\dvar{h}
  \\
   & =t^{\frac{d+1}{2}}v_{n,\partial \Omega}\int_{-\delta_0}^{k_0} h {\e}^{-h^2}\int_0^{\sqrt{\delta_0^2 - h^2}}\int_{\mathbb{S}^{d-2}}{\e}^{- \varrho^2} \varrho^{d-2}\dvar{\varphi}\dvar{\varrho}\dvar{h}.
\end{align*}
Continuing with the two inner integrals,
\begin{align*}
  \int_0^{\sqrt{\delta_0^2 - h^2}}\int_{\mathbb{S}^{d-2}}{\e}^{-\varrho^2}\varrho^{d-2}\dvar{\varphi}\dvar{\varrho} & = \frac{|\mathbb{S}^{d-2}|}{2}\int_0^{\delta_0^2 - h^2}{\e}^{-s}s^{d/2-3/2}\dvar{s} \\
                                                                                                                 & = \frac{|\mathbb{S}^{d-2}|}{2}\gamma\left(\frac{d-1}{2},\delta_0^2-h^2\right).
\end{align*}
Using this expression in the full integral and applying partial integration in the second equality below yields
\begin{align*}
  III & =t^{\frac{d+1}{2}}v_{n,\partial \Omega}\frac{|\mathbb{S}^{d-2}|}{2}\int_{-\delta_0}^{k_0}{\e}^{-h^2}h \gamma\left(\frac{d-1}{2},\delta_0^2-h^2\right)\dvar{h}                                                            \\
      & = t^{\frac{d+1}{2}}v_{n,\partial \Omega}\frac{|\mathbb{S}^{d-2}|}{2}\bigg(\frac{1}{2}\left[ -{\e}^{-h^2}\gamma\left(\frac{d-1}{2},\delta_0^2-h^2\right) \right]^{k_0}_{-\delta_0}                                        \\
      & \qquad - \frac{1}{2}{\e}^{-\delta_0^2}\int_{-\delta_0}^{k_0}(\delta^2-h^2)^{(d-3)/2}h\dvar{h}\bigg)                                                                                                            \\
      & = t^{\frac{d+1}{2}}v_{n,\partial \Omega}\frac{|\mathbb{S}^{d-2}|}{2}\left(\frac{1}{2}{\e}^{-k_0^2}\gamma\left(\frac{d-1}{2},\delta_0^2-k_0^2\right)  + {\e}^{-\delta_0^2}\frac{(\delta_0^2 - k_0^2)^{(d-1)/2}}{d-1}\right).
\end{align*}
Thus, we know that
\begin{equation}\label{eq:III_bound}
  III =t^{\frac{d+1}{2}}v_{n,\partial \Omega} \frac{p|\mathbb{S}^{d-2}|}{2}\left({\e}^{-\delta_0^2}\frac{(\delta_0^2 - k_0^2)^{(d-1)/2}}{d-1} +\frac{1}{2}{\e}^{-k_0^2}\gamma\left(\frac{d-1}{2},\delta_0^2-k_0^2\right)\right).
\end{equation}

We now address the integral $II$ defined in \cref{eq:thm1_II_III}, which means we need to calculate
\begin{equation*}
  J \coloneqq \int_{B_{R}( x) \cap \Omega_i} {\e}^{-\|\hat x - y\|^2/t} p\dy.
\end{equation*}
After a change cylindrical coordinates as for $III$, we rewrite this integral as
\begin{equation*}
  J = p\int_{-\delta_0}^{k_0} {\e}^{-h^2}\logamma{\frac{d-1}{2}}{\delta_0^2 - h^2}\dvar{h}.
\end{equation*}
We can immediately bound $J$ from above by
\begin{equation}\label{eq:J_upperbound}
  p\Gamma\left(\frac{d-1}{2}\right)\int_{-\delta_0}^{k_0} {\e}^{-h^2}\dvar{h} \leq  p\Gamma\left(\frac{d-1}{2}\right)\int_{-\infty}^{\infty}{\e}^{-h^2}\dvar{h} = p\Gamma\left(\frac{d-1}{2}\right)\sqrt{\pi}.
\end{equation}
Now we bound $J$ from below: Since the integrand is positive, we can without loss of generality assume that $k_0 < 0$. Then a change of variables $h = - \sqrt{\delta_0^2 - y}$ yields that
\begin{align*}
  J & \geq p{\e}^{-\delta_0^2}\int_{0}^{\delta_0^2 - k_0^2} {\e}^{y} \logamma{\frac{d-1}{2}}{y}\frac{1}{2\sqrt{\delta_0^2-y}}\dy \\
    & \geq p{\e}^{-\delta_0^2} \frac{1}{2\delta_0} \int_0^{\delta_0^2-k_0^2}{\e}^y \logamma{\frac{d-1}{2}}{y}\dy.
\end{align*}
Using partial integration above we then get
\begin{align*}
  J & \geq \frac{p{\e}^{-\delta_0^2}}{2\delta}\left[{\e}^y\logamma{\frac{d-1}{2}}{y} - \frac{y^{\frac{d-1}{2}}}{\frac{d-1}{2}} \right]_0^{\delta_0^2 - k_0^2}                       \\
    & =\frac{p{\e}^{-\delta_0^2}}{2\delta_0}\left({\e}^{\delta_0^2 - k_0^2} \logamma{\frac{d-1}{2}}{\delta_0^2 - k_0^2} - \frac{2(\delta_0^2 - k_0^2)^{\frac{d-1}{2}}}{d-1}\right).
\end{align*}
Simplifying further gives us
\begin{equation}\label{eq:J_lowerbound}
  J \geq \frac{p}{2\delta_0}\left({\e}^{-k_0^2}\logamma{\frac{d-1}{2}}{\delta_0^2-k_0^2}-{\e}^{-\delta_0^2}\frac{2(\delta_0^2-k_0^2)^{\frac{d-1}{2}}}{d-1}\right).
\end{equation}
Thus, equation \cref{eq:III_bound} and the bounds in \cref{eq:J_lowerbound} and \cref{eq:J_upperbound} proves the theorem.
  {\flushright \qed}

\section{Proofs of \cref{sec:hypothesis}}\label{sec:proof:hypothesis}
\subsection{Proof of \cref{thm:hypothesis}}
Set $R = 1$ and $r_0 = 1/\sqrt{t}$. By the assumption $t < \min\{1/4, 3/(2d)\}$, we have $r_0 > \max\{2, \sqrt{2d/3}\}$, so the hypotheses of \cref{thm:explicit:intersection} are satisfied. Under $H_0$, we have $B_2(x_0) \cap \Omega = B_2(x_0) \cap \Omega_i$ and $\partial \Omega_i \cap B_{2}(x_0) = \emptyset$.

For any $x \in B_1(x_0) \cap \Omega_i$, since $x \in \Omega_i$ we have $\hat x = x$ and thus $\theta = 0$. Therefore $\sin \theta = 0$, and \cref{thm:explicit:intersection} gives
\begin{equation*}
  L_t^i f(x) = t^{\frac{d+1}{2}} B(x) {\e}^{-r_0^2} = t^{\frac{d+1}{2}} B(x) {\e}^{-1/t}.
\end{equation*}
By the bound on $B$ from \cref{thm:explicit:intersection}, we have $|B(x)| \leq \frac{d+1}{4} r_0^{d} p|\mathbb{S}^{d-1}|$.

Furthermore, under $H_0$, if $j \neq i$ then $\Omega_j \cap B_2(x_0) = \emptyset$. Thus for $x \in B_1(x_0)$ and $y \in \bigcup_{j \neq i} \Omega_j$, we have $\|x-y\| \geq \|y-x_0\| - \|x-x_0\| \geq 2-1=1$. Since $\int_{\bigcup_{j \neq i} \Omega_j} p(y) \dvar{y} \leq 1$, this implies
\begin{equation*}
  \left|\sum_{j \neq i} L_t^j f(x)\right| \leq \left( \max_{\rho \geq 1} {\e}^{-\rho^2/t}\rho \right) \int_{\bigcup_{j \neq i} \Omega_j} p(y) \dvar{y}
  \leq {\e}^{-1/t}.
\end{equation*}

Combining the above bounds, for $x \in B_1(x_0)$ we have
\begin{equation*}
  |L_t f(x)| \leq t^{\frac{d+1}{2}} |B(x)| {\e}^{-1/t} + {\e}^{-1/t} =: M(t).
\end{equation*}
Using the bound on $|B(x)|$ and $r_0 = 1/\sqrt{t}$, we obtain $M(t) \leq \frac{p(d+1)|\mathbb{S}^{d-1}|}{4} \sqrt{t}\, {\e}^{-1/t} + {\e}^{-1/t}$. For $t \leq 1$, we have $\sqrt{t} \leq 1$, so
\begin{equation*}
  M(t) \leq C_d \, {\e}^{-1/t},
\end{equation*}
where $C_d = \frac{p(d+1)|\mathbb{S}^{d-1}|}{4} + 1$. Requiring $M(t) \leq \delta/2$ and squaring gives $C_d^2 {\e}^{-2/t} \leq t\log(2n/\alpha)/(4\e(n-1))$. Since $\log t \leq 0$ for $t \leq 1$, taking logarithms yields the sufficient condition $2/t \geq \log(4C_d^2 \e(n-1)/\log(2n/\alpha))$. Thus, the bound on $t$ ensures $M(t) \leq \delta/2$.

Setting $\epsilon = \delta/2$ and applying \cref{thm:finite:sample}, we obtain
\begin{equation*}
  \P\left(\max_m \left|L_{n,t} f(X_m) - \frac{n-1}{n}L_t f(X_m)\right| > \frac{\delta}{2}\right) \leq 2n \exp \left ( -\frac{\e (n-1) \delta^2}{t} \right ).
\end{equation*}
By the definition $\delta = \sqrt{\frac{t}{\e(n-1)}\log \frac{2n}{\alpha}}$, the right-hand side equals $\alpha$.

On the complement of this event, for any $m$ with $X_m \in B_1(x_0)$:
\begin{equation*}
  |L_{n,t} f(X_m)| \leq \left|L_{n,t} f(X_m) - \frac{n-1}{n}L_t f(X_m)\right| + \frac{n-1}{n}|L_t f(X_m)| < \frac{\delta}{2} + M(t) \leq \delta.
\end{equation*}
Thus $\P(T > \delta) \leq \alpha$, completing the proof.
{\flushright \qed}

\subsection{Proof of \cref{thm:power}}
Under hypothesis $H_1'$ we have from \cref{thm:explicit:intersection} that for $r_0 = 1/\sqrt{t}$ that
\begin{equation*}
  L_t^i f(x) = t^{\frac{d+1}{2}} \left (A(d,r_0,\theta_i(x)) v_{n,\Omega_i} r_i(x) \sin \theta_i(x) {\e}^{-r_i(x)^2 \sin^2 \theta_i(x)} + B(x) {\e}^{-r_0^2} \right ),
\end{equation*}
whenever $x \in B_1(x_0)$. Also note that $\theta_i(x) = 0$ if $x \in \Omega_i$, see \cref{fig:thm2}.

The power of rejecting $H_0$ when $H_1'$ is true can be calculated by
\begin{equation*}
  \P(T > \delta).
\end{equation*}
Define the events
\begin{align*}
  A & := \{\max_m |L_{n,t} f(X_m)- L_t f(X_m)| < \delta\},
  \\
  B & := \{\max_m |L_{t} f(X_m)|\chi_{B_1(x_0)}(X_m) > 2 \delta\}.
\end{align*}
Then, for any index $m$ satisfying $A\cap B$ we have
\begin{equation*}
  \max_m |L_{n,t} f(X_m)| \geq \max_m (|L_t f(X_m)| - |L_{n,t} f(X_m) - L_t f(X_m)|) \geq \delta,
\end{equation*}
which implies, together with the union bound, that
\begin{align*}
  \P(T > \delta)
   & \geq
  \P(A \cap B) \\
   & =
  1- \P(A^c \cup B^c)
  \\
   & \geq
  1 - \P(A^c) - \P(B^c).
\end{align*}

We can bound $\P(A^c)$ immediately using \cref{thm:finite:sample}:
\begin{equation*}
  \P(\max_m |L_{n,t} f(X_m)- L_t f(X_m)| > \delta) \leq \alpha.
\end{equation*}

We next bound $\P(B^c)$, using \cref{thm:explicit:intersection} and the definition of $\delta$:
\begin{align*}
  \P(B^c)
   & \leq \P(\max_m |L_t^1 f(X_m)|\chi_{B_1(x_0)\cap \Omega_2}(X_m) < 2\delta + L_t^2 f(X_m)\chi_{B_1(x_0)\cap \Omega_2}(X_m))
  \\
   & \leq
  \P(\max_m |L_t^1 f(X_m)|\chi_{B_1(x_0)\cap \Omega_2}(X_m) < 3\delta)
  \\
   & = \P(|L_t^1 f(X_m)| \chi_{B_1(x_0) \cap \Omega_2}(X_m) < 3\delta)^n.
\end{align*}

Applying \cref{thm:explicit:intersection}, together with our bound on $t$ which ensures that $B(x) {\e}^{-r_0^2} \leq \delta$, we obtain that for $x \in B_1(x_0) \cap \Omega_2$
\begin{equation*}
  |L_t^1 f(x)| \geq t^{\frac{d+1}{2}} \frac{p}{2}\pi^{d/2} |v_{n,\Omega_1}| \frac{\|x-x_0\|}{\sqrt{t}} |\sin \theta_1(x)| {\e}^{-\|x-x_0\|^2 \sin^2 \theta_1(x)/t} - \delta .
\end{equation*}
Denote $G(x) = x {\e}^{-x^2}$, then $G^{-1}(y) = \sqrt{-\W_{0}(-2y^2)/2}$ and $G^{-1}(y) = \sqrt{-\W_{-1}(-2y^2)/2}$ are the two solutions to $G(x) = y$, and $\W$ is the Lambert $\W$ function and the subindex denotes the two different branches.
Thus, we can estimate
\begin{equation*}
  \P(|L_t^1 f(X_m)| \chi_{B_1(x_0) \cap \Omega_2}(X_m) < 3\delta)
  \leq
  \P(G(\|X_m-x_0\| \sin \theta_1/ \sqrt{t}) \chi_{B_1(x_0) \cap \Omega_2} < \frac{\delta}{t^{\frac{d+1}{2}}C})
\end{equation*}
where $C = p\pi^{d/2} |v_{n,\Omega_1}|/8$.
Using the inverse of $G$ as above we get the following terms
\begin{align*}
  P_1 & :=\P\left (\|X_m-x_0\| < \frac{\sqrt{-t \W_{0}\left (-2\frac{\delta^2}{t^{d+1} C^2} \right )}}{\sqrt{2}\sin \theta_1}; X_m \in B_1(x_0) \cap \Omega_2 \right ),
  \\
  P_2 & :=\P\left (\|X_m-x_0\| > \frac{\sqrt{-t \W_{-1}\left (-2\frac{\delta^2}{t^{d+1} C^2} \right )}}{\sqrt{2}\sin \theta_1}; X_m \in B_1(x_0) \cap \Omega_2 \right ), \\
  P_3 & :=\P(X_m \in (B_1(x_0) \cap \Omega_2)^c).
\end{align*}
For $P_1$ we note that $\W_0$ satisfies
\begin{equation*}
  -\W_0(-\varrho) < \e \varrho,
\end{equation*}
for $0 \leq \varrho \leq 1$, as such
\begin{equation*}
  P_1 \leq \P\left (\|X_m-x_0\| < \frac{\sqrt{\e} \delta}{t^{d/2} C \sin \theta_1}\right ).
\end{equation*}
Thus, if we require that
\begin{equation*}
  \frac{\sqrt{\e} \delta}{t^{d/2} C \sin \theta_1} \leq 1/4,
\end{equation*}
which is the second condition in~\cref{eq:power}, we have that $P_1 \leq \P(X_m \in B_{1/4}(x_0))$.
Secondly, we know that
\begin{equation*}
  -\W_{-1}(-\varrho) > \log 1/\varrho,
\end{equation*}
for $0 < \varrho < 1$, and as such
\begin{equation*}
  P_2 \leq \P\left (\|X_m-x_0\| > \frac{\sqrt{t \log\left (\frac{t^{d+1} C^2}{2\delta^2} \right )}}{\sqrt{2}\sin \theta_1}; X_m \in B_1(x_0) \cap \Omega_2 \right ).
\end{equation*}
Which if
\begin{equation*}
  t \left ((d+1)\log(1/t) + \log \left (\frac{2^3}{C^2} \right ) \right ) \leq 2(1-\sin^2 \theta_1),
\end{equation*}
which is the first condition in~\cref{eq:power}, is bounded as
\begin{equation*}
  P_2 \leq \P (X_m \in (B_1(x_0) \setminus B_{1/2}(x_0)) \cap \Omega_2).
\end{equation*}
In conclusion, we have that for large enough $n$, $P_1+P_2+P_3 < 1$.
Assembling the above we have proven our theorem.
  {\flushright \qed}

\section{Proof of \cref{thm:general} and \cref{thm:noisy}}\label{sec:proof:general}

\subsection{Proof of \cref{thm:general}}
We begin by splitting up the domain $\Omega_i$:
\begin{equation}\label{eq:general:integral}
  \begin{split}
    L^i_tf(x) & = \int_{\Omega_i} K_t(x,y)(f(x)-f(y))p \dy                      \\
              & = \int_{\Omega_i \cap B_R(x)}K_t(x,y)(f(x)-f(y))p\dy
    \\
              & \quad + \int_{\Omega_i \setminus B_R(x)}K_t(x,y)(f(x)-f(y))p\dy \\
              & = I + II.
  \end{split}
\end{equation}
We first note that
\begin{equation}\label{eq:outside_integral}
  II = \int_{\Omega_i \setminus B_R(x)}K_t(x,y)(f(x)-f(y))p\dy \leq {\e}^{-R^2/t} \diam(\Omega).
\end{equation}

To estimate $I$ we will make a change of variables to the tangent space at $x_0$ and use arguments similar to those in the proof of \cref{thm:explicit:intersection}. Specifically, let  $\pi: \Omega_i \cap B_R(x) \to T_{\Omega_i,x_0}\cap B_R(x)$ be the projection map, and $\alpha =   \pi^{-1} \circ i: \R^d \cap B_R(0) \to \Omega_i \cap B_R(x)$ a coordinate chart as in~\eqref{eq:chart}. We will use $\alpha$ to integrate over $T_{\Omega_i,x_0}$.

To simplify notation, we will use $\hat x$ and $\hat y$ to denote both $\pi(x),\pi(y) \in \R^N$, and sometimes implicitly assume the projection $i^{-1}$ such that $\hat x, \hat y \in \R^d$. The space in which these points lie should be clear from context.

Before making the coordinate change, we find bounds relating $K(x,y)$ to $K(x,\hat y)$:
We recall that $K_t(x,y) = {\e}^{\frac{-\norm{x-y}^2}{t}}$. By the triangle inequality and $\norm{y - \hat{y}} \leq 4LR^2$ from~\eqref{eq:manifold:bound1},
\begin{equation*}
  \big| \norm{x-y} - \norm{x-\hat{y}} \big| \leq \norm{y - \hat{y}} \leq 4LR^2.
\end{equation*}
Since $y \in B_R(x)$ we have $\norm{x-y} \leq R$ and $\norm{x-\hat{y}} \leq R + 4LR^2$, so using $|a^2 - b^2| = |a-b|(a+b)$,
\begin{equation*}
  |\norm{x-y}^2 - \norm{x-\hat{y}}^2| \leq 4LR^2(2R + 4LR^2) = 8LR^3 + 16L^2R^4.
\end{equation*}
This yields
\begin{equation*}
  {\e}^{-(8LR^3 + 16L^2R^4)/t} K_t(x,\hat{y}) \leq K_t(x,y) \leq {\e}^{(8LR^3 + 16L^2R^4)/t}K_t(x,\hat{y}).
\end{equation*}
Using $|{\e}^x - 1| \leq |x|{\e}^{|x|}$, we get
\begin{equation}\label{eq:CK}
  |K_t(x,y) - K_t(x,\hat y)| \leq \frac{8LR^3 + 16L^2R^4}{t}{\e}^{(8LR^3 + 16L^2R^4)/t} K_t(x,\hat y).
\end{equation}
Replacing $K_t(x,y)$ with $K_t(x,\hat y)$ in $I$ we get
\begin{equation}\label{eq:inside_signal}
  I = \int_{\Omega_i \cap B_R(x)} K_t(x,\hat{y})(f(x)-f(y))p\dy + E_1,
\end{equation}
and using \cref{eq:CK} it holds that
\begin{equation}\label{eq:inside_noise}
  |E_1| \leq \frac{8LR^3 + 16L^2R^4}{t}{\e}^{(8LR^3 + 16L^2R^4)/t} R \left|\int_{\Omega_i \cap B_R(x)} K_t(x,\hat{y})p\dy\right|.
\end{equation}

We now decompose the integral of the first term in \cref{eq:inside_signal} as follows
\begin{align}\label{eq:proof:linear}
  \notag & \int_{\Omega_i \cap B_R(x)} K_t(x,\hat{y})(f(x)-f(y))p\dy
  =  \int_{\Omega_i \cap B_R(x)} K_t(x,\hat{y})(f(x)-f(\hat{y}))p\dy                  \\
         & \qquad + \int_{\Omega_i \cap B_R(x)} K_t(x,\hat{y})(f(\hat{y})-f(y))p\dy =
  I_1 + I_2
\end{align}
The quantity $I_2$ will be treated like an error term.
Using \cref{eq:manifold:bound1} we see that
\begin{equation*}
  |I_2| \leq \int_{\Omega_i \cap B_R(x)} K_t(x,\hat{y})L\norm{\hat{y}-x_0}^2 p\dy.
\end{equation*}
Now we make a coordinate change with $\alpha$ and use the bound on the volume form in \cref{eq:manifold:bound2} to get
\begin{equation*}
  \begin{split}
     & \int_{\Omega_i \cap B_R(x)} K_t(x,\hat{y})L\norm{\hat{y}-x_0}^2 p\dy                                      \\
     & \qquad  \leq LR^2\int_{T_{\Omega_i,x_0}\cap B_R(x)}K_t(x,\hat{y})(1+L\norm{x_0-\hat{y}}^2)p \dvar{\hat y} \\
     & \qquad \leq LR^2(1+L4R^2)\int_{T_{\Omega_i,x_0}\cap B_R(x)}K_t(x,\hat{y})p\dvar{\hat y} .
  \end{split}
\end{equation*}
The RHS of the above display can be handled similarly to \cref{eq:II}, which means
\begin{equation*}
  \begin{split}
    |I_2| & \leq  LR^2(1+L4R^2)\left| \mathbb{S}^{d-1}\right|t^{d/2}p\Gamma(d/2)
    =  LR^2(1+L4R^2)t^{d/2}2p\pi^{d/2}.
  \end{split}
\end{equation*}

We proceed now with $I_1$ from \cref{eq:proof:linear}, which we want to estimate as accurately as possible. Using the coordinate change $\alpha$ and \cref{eq:manifold:bound2} we write
\begin{align} \label{eq:proof:I1}
  \notag I_1
   & = {\e}^{-r^2 \sin^2 \theta}\int_{\Omega_i \cap B_R(x)}K_t(\hat x,\hat y)(f(x)-f(\hat y))p \dy                         \\
   & = {\e}^{-r^2 \sin^2 \theta}C_1(x)\int_{T_{\Omega_i,x_0} \cap B_R(x)}K_t(\hat x,\hat y)(f(x)-f(\hat y))p\dvar{\hat y},
\end{align}
where $C_1(x) > 0$ is such that $|C_1 -1| \leq (1+L4R^2)$.

The integral on the right in \cref{eq:proof:I1} is exactly $II$ from \cref{eq:thm1_II_III}, which we compute as in \cref{eq:II}:
\begin{equation}\label{eq:proof:I1+}
  \int_{T_{\Omega_i,x_0} \cap B_R(x)}K_t(\hat x,\hat y)(f(x)-f(\hat y))p\dvar{\hat y}
  = A(d,r_0,\theta)v_{n,\Omega_i}t^{\frac{d+1}{2}}r\sin\theta,
\end{equation}
where $A(d,r_0,\theta)$ is as in~\eqref{eq:A}. Now, from \cref{eq:proof:linear,eq:proof:I1,eq:proof:I1+} we have
\begin{equation*}
  \begin{split}
    I_1 + I_2 =  C_1(x) A(d,r_0,\theta)v_{n,\Omega_i}t^{\frac{d+1}{2}}r\sin\theta {\e}^{-r^2 \sin^2\theta}  + C_2(x)LR^2(1+L4R^2)t^{d/2}2p\pi^{d/2},
  \end{split}
\end{equation*}
for a function $|C_2(x)| \leq 1$.
This combined with the split in~\eqref{eq:inside_signal} and~\eqref{eq:inside_noise} gives us
for another function $|C_3(x)| \leq 1$
\begin{align*}
  I = & I_1 + I_2 + E_1 = C_1(x) A(d,r_0,\theta)v_{n,\Omega_i}t^{\frac{d+1}{2}}r\sin\theta {\e}^{-r^2 \sin^2\theta}
  \\
      & + C_3(x) C_{L,R} t^{d/2}2p\pi^{d/2}
\end{align*}
where
\begin{equation*}
  C_{L,R} := \frac{(8LR^3 + 16L^2R^4)R}{t}{\e}^{(8LR^3 + 16L^2R^4)/t}+LR^2(1+4LR^2).
\end{equation*}
Finally, the above and~\eqref{eq:outside_integral} finishes the proof.
  {\flushright \qed}

\subsection{Proof of \cref{lem:error}}\label{sec:proof:error}
First applying \cref{thm:general} to $L_t^if(x)$, and then using the $(L,2R)$ regularity of $\Omega_i$, we use \cref{eq:manifold:bound1} to get
\begin{equation*}
  |r \sin \theta| = \frac{\|x-\hat x\|}{\sqrt{t}} \leq \frac{L R^2}{\sqrt{t}} = L r_0^2 \sqrt{t}.
\end{equation*}
Thus, the signal term from \cref{thm:general} satisfies
\begin{equation*}
  \left|t^{\frac{d+1}{2}} \widehat A v_{n,\Omega_i} r \sin \theta \, {\e}^{-r^2 \sin^2 \theta}\right| \leq t^{\frac{d+1}{2}} \widehat A \cdot L r_0^2 \sqrt{t} = t^{\frac{d+2}{2}} \widehat A L r_0^2.
\end{equation*}
Combining with the error term $t^{(d+2)/2}\widetilde{C}_{L,r_0}$ from \cref{thm:general}, we obtain the bound on $E_{L,r_0}$.
  {\flushright \qed}

\subsection{Proof of \cref{thm:noisy}}\label{sec:proof:noisy}
To simplify notation, let $h_j = x-X_j$.
\begin{align} \label{eq:thm:stochastic}
  \notag \E_\epsilon L_{n,t}f(x)
   & = \frac{1}{n}\sum_{j=1}^n \E_\epsilon K_t(x,X_j)(f(x)-f(X_j+\epsilon_j))                   \\
   & = \frac{1}{n}\sum_{j=1}^n \E_\epsilon {\e}^{-\norm{h_j-\epsilon_j}^2/t}v\cdot(h_j-\epsilon_j)
\end{align}
Let us compute a single term in the sum in \cref{eq:thm:stochastic}: Since the expectation is w.r.t $\epsilon$ we can treat $h = h_j$ as fixed, and then algebraic manipulations give us for $z \sim \mathcal{N}(0,\sigma^2 I)$
\begin{align*}
         & \notag \E_{z}{\e}^{-\norm{h+z}^2/t}(h+z)                                                                                                                              \\
         & \quad=(2\pi\sigma^2)^{-N/2}\int_{\mathbb{R}^N} {\e}^{-\norm{h+z}^2/t}{\e}^{-\norm{z}^2/(2\sigma^2)}v\cdot(h+z)\dz                                                        \\
  \notag & \quad=(2\pi \sigma^2)^{-N/2}\int_{\mathbb{R}^N} {\e}^{-\left(\norm{h}^2/t + 2 \langle h, z \rangle / t + \norm{z}^2/t + \norm{z}^2 /(2\sigma^2)\right)}v\cdot(h+z)\dz \\
         & \quad= (2\pi \sigma^2)^{-N/2} {\e}^{-\frac{\norm{h}^2}{2 \sigma^2 + t} } \int_{\mathbb{R}^N} {\e}^{-\frac{1}{kt}\norm{z+kh}^2} v\cdot(h+z) \dz.
\end{align*}
In the second to last step we completed the square and  used $k=\frac{2\sigma^2}{2\sigma^2+t}$.
This last integral can be viewed as the expectation
\begin{equation*}
  (\pi kt)^{-N/2} \int_{\mathbb{R}^N} {\e}^{-\frac{1}{kt}\norm{z+kh}^2} v\cdot(h+z)\dz = \E_X [(h+X)] = (1-k)v\cdot h
\end{equation*}
where $X \sim \mathcal{N}(-kh,I\frac{kt}{2})$. Then we can conclude that
\begin{equation*}
  \E_{z}{\e}^{-\norm{h+z}^2/t}v \cdot (h+z) = v \cdot h \frac{t^{N/2+1}}{(2 \sigma^2 + t)^{N/2+1}} {\e}^{-\frac{\norm{h}^2}{2 \sigma^2 + t} }.
\end{equation*}
{\flushright \qed}

\section{Proof of \cref{thm:finite:sample}}\label{sec:proof:finite}
First, using the union bound we get
\begin{equation} \label{eq:unionbound}
  \begin{split}
     & P \bigg( \max_i \left |L_{n,t}f(X_i) - \frac{n-1}{n} L_t f(X_i) \right | > \epsilon \bigg )                      \\
     & \qquad \leq \sum_{i=1}^n P \left ( \left |L_{n,t}f(X_i) - \frac{n-1}{n} L_t f(X_i) \right | > \epsilon \right ).
  \end{split}
\end{equation}
Using the definitions of $L_{n,t}$ and $L_{t}$, see \cref{eq:Lnt,eq:Lt}, and using that the random variables $X_1,\dots,X_n$ are i.i.d., we can replace each $X_i$ by $X_1$ in each term in the summand of \cref{eq:unionbound}. Let $Z$ be an independent copy of $X_1$. Then each summand in \cref{eq:unionbound} equals
\begin{equation}\label{eq:prob1}
  \begin{split}
    P \bigg( & \bigg | \frac{1}{n} \sum_{j=1}^n K_t(X_1,X_j) (f(X_1)-f(X_j)) \\
             & \quad - \frac{n-1}{n}\E_Z[K_t(X_1,Z)(f(X_1)-f(Z))]\bigg |
    > \epsilon \bigg ).
  \end{split}
\end{equation}
To simplify notation, we denote
\begin{equation*}
  W_i(x) = K_t(x,X_i)(f(x)-f(X_i))  \quad\mathrm{and}\quad
  Y_i(x) = W_i(x)-\E_{X_i}[W_i(x)].
\end{equation*}
We now rewrite \cref{eq:prob1} as
\begin{equation*}
  P \left ( \left |\frac{1}{n-1}\sum_{i=2}^n Y_i(X_1)\right | > \frac{n}{n-1}\epsilon \right ).
\end{equation*}
Now by the tower property we have that
\begin{equation*}
  P \left ( \left |\frac{1}{n-1}\sum_{i=2}^n Y_i(X_1)\right | > \frac{n}{n-1}\epsilon \right ) = \E\left [P \left ( \left |\frac{1}{n-1}\sum_{i=2}^n Y_i(X_1)\right | > \frac{n}{n-1}\epsilon  \mid X_1 \right ) \right ]
\end{equation*}
In order to use Hoeffding's inequality we need to show that $Y_i(x)$ is a bounded random variable for all $x \in \Omega$.
First,
\begin{align*}
  W_i(x) & = K_t(x,X_i)(f(x)-f(X_i))
  =
  {\e}^{-\norm{x-X_i}^2/t}v \cdot (x-X_i)
  \leq
  {\e}^{-\norm{x-X_i}^2/t} \|x-X_i\|
  \\
         & \leq
  \sup_{\rho} {\e}^{-\rho^2} \rho \sqrt{t}
  \leq \sqrt{\frac{t}{2\e}}.
\end{align*}
Now Hoeffdings inequality states that (where $C_n = \frac{n}{n-1}$)
\begin{equation*}
  \P \left ( \left | \frac{1}{n-1} \sum_{i=2}^n Y_i \right | > C_n \epsilon \mid X_1 \right ) \leq 2 \exp\left ( -\frac{4\e (n-1) C_n^2 \epsilon^2}{t} \right )
\end{equation*}
and the proof is complete after taking expectations.
  {\flushright \qed}

\section{Construction of the neural network dataset}\label{sec:constraint}

As explained in \cref{sec:num:neural}, we consider neural networks of the form given in \cref{sec:neural}, where $k=3$ and $W,W^* \in \mathbb{R}^{6}$, $a_1=-a_2=-a_3$, all three weights $w^\ast_i$ are the same. We will also set $\mathbf{B} = [-10,10]^{6}$.

In order to create our data-set we begin by setting up a regression problem, i.e. we sample $\mathcal{D}=\{x_1,\ldots,x_{100}\}$ from the uniform density on the unit circle in $\R^2$, and our data pairs for the regression is then $(x_i,f_{W^\ast}(x_i))$. We are interested in the zero set of the neural network with respect to the dataset $(x_i,f_{W^\ast}(x_i))$, i.e. 
\begin{equation*}
  \widehat{\Omega}_\delta := \{W \in \mathbf{B}: |f_W(x_i) - f_{W^\ast}(x_i)| < \delta, \quad x_i \in \mathcal{D}  \}.
\end{equation*}

To try to limit computations, we use interval constraint propagation, and more specifically forward-backward propagation in the computational tree of the expressions $|f_W(x_i) - f_W^{*}(x_i) | < \delta$, where $\delta \approx 10^{-16}$ and $x_i \in \mathcal{D}$ to produce a paving (a true covering of boxes). For more information about these methods, see \cite{jaulin2014introduction,jaulin2001guaranteed,david_p_sanders_2022_5888392}.
The resulting paving is a set of boxes $\mathbb{X}^+$, where each box in $\mathbb{X}^+$ has a maximum width of $0.01$.
We will again reiterate here that the choice of $k=3$ and box-size $0.01$ is what we found to be the largest $k$ and smallest box size that still allowed us to compute the zero set in a reasonable time, and we ended up with roughly $10^7$ boxes in $\mathbb{X}^+$.

To construct a point of the zero set we will simply choose the centroid of each box as a representative. By construction, we know that such a centroid is at most $0.005\sqrt{6}$ distance away from a point in $\widehat{\Omega}_\delta$. In our experiment, this leads to a dataset of $12,753,597$ points $X$ that are within a small distance ($0.005\sqrt{6}$) of $\widehat{\Omega}_\delta$ such that if we put a box around each with side-length $0.01$, then the union of these boxes will cover $\widehat{\Omega}_\delta$. This is in contrast with other search methods, like a grid search (which would require $\sim 10^{19}$ boxes) or a random search (hitting probability is $10^{-13}$) that cannot guarantee the coverage of $\widehat{\Omega}_\delta$.

\end{document}